\documentclass[runningheads]{llncs}
\usepackage{graphicx}
\usepackage{amsmath,amssymb}
\usepackage{mathtools}
\usepackage{microtype}
\usepackage{booktabs}
\usepackage{multirow}
\usepackage{xspace}
\usepackage{tikz}
\usepackage{hyperref}
\usepackage[linesnumbered,ruled,vlined]{algorithm2e}
\usetikzlibrary{arrows.meta,positioning,calc,fit,backgrounds}
\newcommand{\floatconts}[3]{%
  \centering
  #2
  #3
  \label{#1}%
}

\newcommand{\ALC}{\ensuremath{\mathcal{ALC}}}

\newcommand{\ALCHOQ}{\ensuremath{\mathcal{ALCHOQ}}}

\newcommand{\ALCHOIQ}{\ensuremath{\mathcal{ALCHOIQ}}}
\newcommand{\SRIQ}{\ensuremath{\mathcal{SRIQ}}}
\newcommand{\SROIQ}{\ensuremath{\mathcal{SROIQ}}}
\newcommand{\ELpp}{\ensuremath{\mathcal{EL}^{++}}}
\newcommand{\sub}{\sqsubseteq}
\newcommand{\sym}[1]{\ensuremath{\mathsf{#1}}}
\newcommand{\Onto}{\mathcal{O}}
\newcommand{\WMC}{\mathsf{WMC}}
\newcommand{\method}{\textsc{Baobab}\xspace}

\begin{document}
\title{Neuro-symbolic
  learning over OWL~2~DL\\
  via consequence-based compilation to differentiable circuits}
\titlerunning{Baobab: NeSy over OWL 2 DL}
%
\author{Olga Mashkova\inst{1}\orcidID{0000-0002-4916-1660}, Asaad Mohammedsaleh\inst{1}\orcidID{0009-0007-3160-8819}, Fernando Zhapa-Camacho\inst{1}\orcidID{0000-0002-0710-2259} and Robert Hoehndorf\inst{1}\orcidID{0000-0001-8149-5890}}
\authorrunning{Mashkova et al.}

\institute{Computer Science Program, Computer, Electrical, and
  Mathematical Sciences \& Engineering Division, King Abdullah
  University of Science and Technology, Thuwal 23955, Saudi Arabia\\
\email{\{first\_name.last\_name\}@kaust.edu.sa}}
\authorrunning{Mashkova et al.}
%
\institute{Computer Science Program, Computer, Electrical, and
  Mathematical Sciences \& Engineering Division, King Abdullah
  University of Science and Technology, Thuwal 23955, Saudi Arabia\\
\email{\{first\_name.last\_name\}@kaust.edu.sa}}
\maketitle              
\begin{abstract}
  OWL~2~DL ontologies, grounded in the description logic \SROIQ,
  express large knowledge bases in biomedicine and the Semantic Web.
  Neuro-symbolic (NeSy) learners over description logics either embed
  the ontology in a continuous space, abandoning classical entailment,
  or restrict to the Horn fragment \ELpp, which has a single canonical
  model.  We present \method{}, which compiles a \SROIQ\ ontology
  with a finite ABox into a Sentential Decision Diagram (SDD): it
  saturates a propositional core under a consequence-based calculus
  and instantiates the remaining \SROIQ\ features (nominals, number
  restrictions, and the role axioms) over the active domain.  The
  SDD's evidence-conditioned weighted model count then trains a
  perception network to \emph{recognize real images} under partial
  ABox supervision: on an ontology that exercises every distinctive
  \SROIQ\ feature, a CNN learns to read MNIST digits coupled by a
  successor relation and recovers latent ontology concepts that an
  independent perception leaves at chance.  When the supervision
  admits several ontology-consistent completions, an independent
  perception collapses onto one, a \emph{reasoning shortcut}: we show
  that a mixture indexed by the query's justifications can represent
  the calibrated posterior no independent perception can, and that
  seeding it from the circuit's enumerated completions attains the
  Bayes-optimal posterior on a real-image MNIST task where
  single-WMC and learned mixtures (the BEARS-ensemble hypothesis
  class) do not: to our knowledge the first to characterize and
  mitigate reasoning shortcuts in a non-Horn description logic.  Soundness of the compiler and the representation
  result are machine-checked in Lean~4. Code is available at
  \url{https://github.com/bio-ontology-research-group/baobab}.
\end{abstract}
\section{Introduction}\label{sec:intro}

OWL~2~DL ontologies, formally grounded in the description logic
\SROIQ~\cite{HorrocksKutzSattler2006}, structure knowledge bases in
biomedicine, scientific cataloging, and the Semantic Web.  We ask a
neural network to read raw inputs into the concepts of such an ontology
when most of those concepts carry no direct labels, so the network must
infer them from the ontology's logical structure.  The neuro-symbolic
(NeSy) community has produced two broad families of OWL-aware learners.  \emph{Embedding-based} methods such as EL
Embeddings~\cite{Kulmanov2019ELEmbeddings},
OWL2Vec*~\cite{Chen2021OWL2Vec}, BoxEL~\cite{Xiong2022BoxEL},
Box$^2$EL~\cite{Jackermeier2024Box2EL}, and
FALCON~\cite{Tang2022FALCON} project an ontology into a continuous
space; they scale to OWL-shaped knowledge graphs but compute no
classical entailment during training.  \emph{Knowledge-compilation} methods such as
DeepProbLog~\cite{Manhaeve2018DeepProbLog}, Semantic
Loss~\cite{Xu2018SemanticLoss}, Semantic Probabilistic
Layers~\cite{Ahmed2022SPL}, Scallop~\cite{Li2023Scallop}, and
NeuPSL~\cite{Pryor2023NeuPSL} compute
the weighted model count (WMC) of the compiled constraint and feed it
to perception as a differentiable signal, but target propositional
formulas or Datalog-shaped programs; the few NeSy methods that take a
description-logic ontology as input target \ALC\ or the OWL~2 Horn
profiles, compiling to a probabilistic
circuit~\cite{Lazzari2026Compiling} or relaxing to fuzzy
semantics~\cite{Wu2022DFALC,Zhao2025DFELpp}.

The gap between the two families of methods is structural.
Knowledge compilation keeps the classical semantics end to end (the
WMC of $\Theta$ under $p_\theta$ is exactly the probability that a
world drawn from the perception's per-atom posterior \emph{satisfies
the ontology}); but the
propositional reduction of \SROIQ\ adds disjunction, classical
negation, qualified cardinality, nominals, and the full family of role
axioms, which a Horn-only saturation calculus does not handle easily.
Embedding-based NeSy buys scaling but loses classical semantics: a
trained embedding can satisfy its loss while violating an entailment
of the ontology.

We present \method{}, a compiler that takes a \SROIQ\ ontology and
produces a Sentential Decision
Diagram (SDD)~\cite{Darwiche2011SDD}: a circuit that scores whether the
perception network's predictions stay logically consistent with the
ontology.  Its weighted model count is the training signal under partial
ABox supervision, so that a real image encoder learns latent ontology
concepts it is never directly shown.

Specifically, we make three contributions.
\textbf{(i)~\SROIQ\ compilation.}  \method{} accepts any OWL 2 DL ontology and compiles it to an SDD, soundly on the grounding-covered fragment,
combining consequence-based saturation over a
propositional core
(\ALCHOQ)~\cite{TenaCucalaCuencaGrauHorrocks2018,TenaCucala2021AIJ}
with finite-domain grounding of the \SROIQ-specific extensions
(nominals, qualified cardinality, and the six role-axiom shapes \ALC\
lacks); an OWL loader~\cite{Lord2023HornedOWL} maps every OWL~2~DL
axiom kind to the compiler's syntax (\S\ref{sec:calc}).
\textbf{(ii)~Latent concept learning of real images.}  The SDD's
evidence-conditioned weighted model count trains a CNN to \emph{recognize
MNIST digit images} through the ontology: under partial ABox supervision
the perception recovers latent concepts it is never directly shown
(digit identity $0.25\!\to\!0.99$) and drives ontology violations toward
zero, where an independent perception sits at chance
(\S\ref{sec:learn}, \S\ref{sec:exp:images}); the same workflow drives a
ResNet through a role-based fragment of the real Pizzaiolo OWL
ontology, recovering its latent pizza classes.
\textbf{(iii)~Reasoning-shortcut awareness in \SROIQ.}  When the
supervision leaves several ontology-consistent completions, a single
independent perception provably cannot place calibrated mass on them, but
a mixture indexed by the query's \emph{justifications} can; realized as
one network (\emph{JustWMC}) and seeded from the completions the circuit
enumerates, it attains the Bayes-optimal calibrated posterior on a
real-image MNIST task where single-WMC and learned mixtures (the
BEARS-ensemble hypothesis class) do not (\S\ref{sec:exp:disjunction}).
A per-axiom comparison shows that $23\%$ of the \SROIQ\ axioms leave
DeepProbLog's Horn fragment (\S\ref{sec:exp:dpl}); \S\ref{sec:lean}
states the formal guarantees behind the compiler and the
representation result.

\section{Background}\label{sec:bg}

A \SROIQ\ knowledge base $(\mathcal{T}, \mathcal{R}, \mathcal{A})$ has
a \emph{TBox} of general concept inclusions $C \sub D$, an \emph{ABox}
of concept and role assertions $C(a)$, $R(a,b)$ with (in)equalities,
and an \emph{RBox} of role hierarchies $R \sub S$, role chains
$R_1 \circ \dots \circ R_k \sub S$, and the eight role characteristics
(transitivity, (a)symmetry, (ir)reflexivity, (inverse-)functionality,
disjointness), over the
standard syntax of~\cite{HorrocksKutzSattler2006}; App.~\ref{apd:semantics} gives the full syntax (concepts close
Boolean connectives, nominals, qualified number restrictions, and
$\exists R.\sym{Self}$) and the Tarskian semantics.

\emph{Consequence-based} (CB) reasoning saturates a set of clauses
under a fixed family of resolution-like inference rules and reads the
entailed concept subsumptions off the resulting closure, rather than
searching for a model as a tableau would.  CB calculi were developed
for \SRIQ~\cite{BateMotikCuencaGrau2016SRIQ} and extended to
\ALCHOIQ~\cite{TenaCucalaCuencaGrauHorrocks2018,TenaCucala2021AIJ},
where clauses are grouped into \emph{contexts} (clause sets anchored
at a concept) between which consequences propagate; they underlie
Sequoia~\cite{TenaCucala2021AIJ} (\SROIQ) and
ELK~\cite{KazakovKrotzschSimancik2014} (the Horn profile).  The
calculus is deterministic and model-free, so the closure is a fixed
propositional object.

\emph{Knowledge compilation} represents a propositional theory as a
circuit on which queries become efficient.  A Sentential Decision
Diagram (SDD)~\cite{Darwiche2011SDD} is a decomposable, deterministic
circuit over a fixed \emph{vtree} (a binary tree over the variables);
its size is worst-case exponential in the theory's treewidth but, once
built, supports linear-time \emph{weighted model counting}.  For atom
weights $w(a) \in [0,1]$, the weighted model count of a theory
$\varphi$ is
\begin{equation}\label{eq:wmc-def}
  \WMC(\varphi \mid w) \;=\;
    \sum_{\alpha \,\models\, \varphi}\;
    \prod_{a:\,\alpha(a)=1} w(a)\,
    \prod_{a:\,\alpha(a)=0} \bigl(1 - w(a)\bigr),
\end{equation}
the total weight of $\varphi$'s models, evaluated on an SDD in one
bottom-up pass over the circuit.

Reading each atom as an independent Bernoulli event with parameter
$w(a)$ is the \emph{distribution semantics}~\cite{Riguzzi2015DISPONTE}:
\eqref{eq:wmc-def} is exactly the probability that a $w$-distributed
world satisfies $\varphi$, the link knowledge-compilation
neuro-symbolic methods exploit to turn per-atom probabilities into a
differentiable training
signal~\cite{Xu2018SemanticLoss,Manhaeve2018DeepProbLog}.

\section{Methods}\label{sec:method}

\subsection{Knowledge compilation algorithm}\label{sec:calc}

The algorithm takes a \SROIQ\ knowledge base
$(\mathcal{T}, \mathcal{R}, \mathcal{A})$ whose ABox ranges over a
finite set of individuals $\Delta$ (the \emph{active domain}), and
returns an SDD over the ground propositional atoms
$\{C(a), R(a, b) : a, b \in \Delta\}$, whose evidence-conditioned
weighted model count is the differentiable training signal of
\S\ref{sec:learn}.  The algorithm runs in five
stages: normalization, structural transformation to DL-clauses,
consequence-based saturation over the propositional \ALCHOQ\
sub-language, grounding of the \SROIQ-extension features over $\Delta$,
and SDD compilation of the resulting CNF
(Figure~\ref{fig:workflow}).

\begin{figure}[t]
\floatconts
{fig:workflow}
{\caption{The \method{} workflow.  \emph{Compile time} (once per
ontology): the five stages above produce an SDD.  \emph{Training
loop}: a shared CNN reads the coupled individuals $a,b$ (real MNIST
digits) into per-atom posteriors that weight the SDD leaves; the
evidence-conditioned WMC is the differentiable loss.}}
{\centering\resizebox{0.60\textwidth}{!}{\input{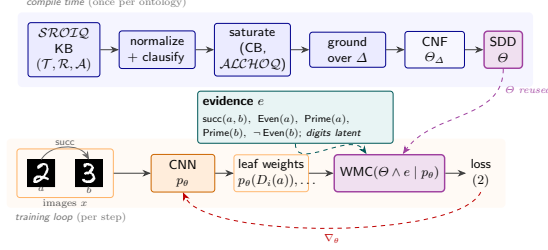}}}
\end{figure}

\emph{Normalization} rewrites every concept into negation-normal form
(NNF) by the standard de~Morgan, quantifier, and cardinality dualities
(App.~\ref{apd:norm}).  The \emph{structural transformation} then
turns the NNF axioms into DL-clauses
$\bigwedge_i B_i \to \bigvee_j H_j$, whose body atoms $B_i$ and head
atoms $H_j$ are concept atoms $C(t)$, role atoms $R(t, t')$, or
equalities $t \approx t'$ over variables, individuals, and unary
function terms $f(t)$, with a fresh name $Q_D$ for each non-atomic
subconcept $D$.  Nominals and qualified number restrictions are not
clausified here but deferred to grounding (Table~\ref{tab:hooks}).

\emph{Saturation} closes the DL-clause set under propositional
\emph{hyperresolution} (resolving body atoms against matching head
atoms under a most-general unifier).  Its role is to eliminate
\emph{Skolem terms}: head existentials and number restrictions
clausify to function terms $f(x)$ naming anonymous witnesses the
finite grounder cannot instantiate, and saturation derives their
\emph{function-free} consequences over the named individuals.  Four
guards keep the closure terminating (tautology
elimination, input-derived bounds on the variables and role atoms per
clause, and a Skolem-term depth bound; exact bounds in
App.~\ref{apd:sat}); the guarded calculus is sound (mechanized for
the \ALCHOQ\ core) but, being bounded, not complete.

\emph{Grounding} instantiates the function-free clauses over the
active domain $\Delta$, replacing variables by individuals and
expanding each deferred nominal, number restriction, and role
characteristic into the propositional clauses listed below.  The
result is a propositional CNF $\Theta_\Delta$ that is sound over the
active domain, and whose models are \emph{exactly} the \SROIQ\
interpretations over $\Delta$ when every existential consequence is
grounding-covered or function-free (Theorem~\ref{thm:ground}).

A construct's consequences reach $\Theta_\Delta$ either by
\emph{grounding} that closes the domain or by \emph{saturation} that
eliminates an existential witness symbolically; each deferred construct
expands over $\Delta$ into boundedly many clauses justified by a
soundness lemma (Theorem~\ref{thm:ground}), and Table~\ref{tab:hooks}
lists every schema.  App.~\ref{apd:size} shows why saturation is
load-bearing exactly for ungrounded existentials and inert otherwise.

Finally, the grounded propositional CNF is compiled to an SDD with
PySDD, and the differentiable WMC layer
of~\cite{Manhaeve2018DeepProbLog} backpropagates the loss through the
SDD's parameterized leaves into the perception network
(Algorithm~\ref{alg:compile} gives the end-to-end workflow).

\subsection{Latent concept learning under partial ABox supervision}\label{sec:learn}

We address \emph{ABox-supervised latent concept learning}.  Each
training instance supplies a perceptual input (an image, or a noisy
feature vector) together with a partial set of ground ABox literals over an
\emph{observable} signature; the concept and role atoms the ontology
entails over a separate \emph{latent} signature are never directly
supervised.  The output is a per-atom posterior over the latent
signature, scored against the generating interpretation; the latent
atoms are tied to the observables only through $\Onto$, so the learner
must propagate evidence through the ontology rather than read labels
off the input.

Formally, let $\Onto$ be a fixed \SROIQ\ ontology,
$p_\theta \colon \mathcal{A} \to [0,1]$ the perception network's
per-atom posterior over the compiled circuit's propositional signature
$\mathcal{A}$, and $e \subseteq \mathcal{A} \times \{T, F\}$ a piece
of supervisory evidence, a partial Boolean assignment over the
example's labels.  The \emph{consistent set} $K_e$ is the set of
complete assignments to $\mathcal{A}$ that satisfy $\Theta \land e$,
where $\Theta$ is the SDD of $\Onto$.  When the supervision
underdetermines the latent identity, $|K_e| > 1$.

We optimize the perception network against the
\emph{evidence-conditioned weighted-model-count loss}
\begin{equation}\label{eq:wmc-loss}
  \mathcal{L}(p_\theta;\, e) \,=\,
  \mathrm{BCE}\bigl(p_\theta(a),\, v\bigr)\bigm|_{(a, v) \in e}
  \,+\, \lambda \cdot
  \bigl(- \log \WMC(\Theta \land e \mid p_\theta)\bigr).
\end{equation}
The label term is binary cross-entropy (BCE) directly supervising the
observed labels; the semantic term is $-\log$ of the WMC of
$\Theta \land e$ under $p_\theta$, the probability that some
completion satisfies the ontology \emph{and} the evidence.  When $e$
determines a single model this collapses to the fully supervised
likelihood; when $|K_e| > 1$ it distributes the gradient across $K_e$
as a credal set, following~\cite{Marconato2023RS,Marconato2024BEARS}.
Both terms are needed: label-only training ignores the ontology, while
conditioning the semantic term on $e$ (rather than on $\Theta$ alone)
ties it to each example's evidence instead of the ontology's globally
most probable assignment.

\eqref{eq:wmc-loss} still commits to a single mode of $K_e$ when
$|K_e| > 1$, and the reason is structural.  A perception that scores
each atom independently induces a product distribution, the
\emph{universal conditionally independent} (UCI) class
$p_\mu^\perp = \prod_i \mu_i^{c_i}(1{-}\mu_i)^{1-c_i}$; the
assignments such a distribution can concentrate on form a
\emph{subcube}: fix some atoms, let the rest vary freely.  A \emph{reasoning shortcut} (RS) is a latent
assignment that satisfies the ontology yet differs from the one that
generated the example~\cite{Marconato2023RS}; with $|K_e| > 1$ the
optimum of \eqref{eq:wmc-loss} over a product distribution is exactly
such a shortcut.  \cite{vanKrieken2025RSIA} show that an independent
perception cannot reproduce the predictions of an RS mixture unless
the mixture's valid completions already form one such subcube: those
consistent with a single \emph{justification} of the query (a minimal
set of atom values forcing the query true).  When the valid completions span several justifications, no
independent perception can spread its mass over them; it collapses
onto an arbitrary one.  Theorem~\ref{thm:rsia} establishes this characterization for \SROIQ.

The converse is previously known: a \emph{mixture} of independent
distributions can be
RS-aware~\cite{vanKrieken2025RSIA,Marconato2024BEARS}.  We adapt it
to \SROIQ: a mixture of UCIs with one component per justification,
$p_\text{Just} = \sum_k \pi_k(x)\,\prod_i
\mu_{k,i}(x)^{c_i}(1{-}\mu_{k,i}(x))^{1{-}c_i}$, represents every RS
mixture (Theorem~\ref{thm:rsia}): it can place mass on completions
drawn from different justifications, which no single product
distribution can cover (the disjunction benchmark of
\S\ref{sec:exp:disjunction} is the canonical case).  We
realize it as one network, the
\texttt{MixtureEncoder}, and call the method \emph{JustWMC}: a shared
body feeds $K$ atom heads $\mu_k$ and a softmax selector $\pi$,
trained end-to-end on a mixture cross-entropy plus the
selector-averaged semantic loss of \eqref{eq:wmc-loss}, with a
stop-gradient KL diversity term (after BEARS) that breaks
head-permutation symmetry and spreads the heads across distinct
$\Theta$-consistent modes (App.~\ref{apd:loss}).  BEARS trains $K$
separate encoders and reports the best against an oracle; JustWMC
returns one calibrated posterior with no oracle, its components
readable off the circuit's justifications (the anchored variant below)
rather than trained.

The heads need not be learned: when the circuit enumerates the query's
justifications, we \emph{seed} each head from one $\Theta$-consistent
completion (fixing its latent logits) and learn only $\pi$:
\emph{anchored} JustWMC, which realizes the positive direction of
Theorem~\ref{thm:rsia} by construction and reaches the calibrated
posterior of \S\ref{sec:exp:disjunction} where the learned mixture,
under the same objective, does not.

\subsection{Formal guarantees}\label{sec:lean}

The compiler and its metatheory are formalized in
Lean~4~\cite{Lean4}; the formal statements and proofs are in
App.~\ref{apd:guarantees} and the module inventory in
App.~\ref{apd:lean}.  Four results support the claims above.
\emph{Saturation is sound} (Theorem~\ref{thm:sat}): every derived
clause is entailed, so over a fixed vtree the compiled SDD is invariant
under saturation on the grounding-covered fragment
(Theorem~\ref{thm:invariance}), and skipping it off that fragment
over-approximates (sound but incomplete).  \emph{Each grounding rule is
faithful} (Theorem~\ref{thm:ground}): $\Theta_\Delta$ is sound over
$\Delta$ and, when every existential is grounding-covered or
function-free, its models are exactly the \SROIQ\ interpretations over
$\Delta$.  \emph{The compiled SDD's weighted model count} equals the
probability the distribution semantics assigns
(Theorem~\ref{thm:wmc}), the quantity the losses of \S\ref{sec:learn}
optimize.  Finally, the RS-awareness characterization
of~\cite{vanKrieken2025RSIA} holds for \SROIQ\
(Theorem~\ref{thm:rsia}).  The development also proves the \ALC\ core
sound and complete and bounds the grounding by a polynomial in
$|\Delta|$.

\section{Experiments}\label{sec:exp}

\subsection{Real-image perception}\label{sec:exp:images}

Unlike prior knowledge-compilation methods for description logics,
which report on hand-built feature vectors~\cite{Lazzari2026Compiling},
we drive the circuit \method{} compiles with two image encoders and
recover ontology atoms that receive no direct supervision.  All
experiments use ten seeds; tables give mean$\pm$std on held-out unseen
individuals, and bold marks a one-sided paired permutation gain
(Holm-corrected per table) at $p<0.05$ (all reach $p<0.01$).

\paragraph{Metrics.}
\emph{Digit accuracy}: per-individual argmax over the five
digit-identity atoms vs.\ the true digit (chance $0.20$).
\emph{Per-atom accuracy} (\emph{latent}, \emph{property},
\emph{topping}): each named atom thresholded at $0.5$, so the easy
negative atoms put it above the argmax.  \emph{Violation}:
fraction of examples whose MAP decode (latent atoms argmaxed, evidence
clamped) has no model ($\WMC=0$).  \emph{Mode-coverage TV}: total
variation to uniform over the $\Theta$-consistent modes ($0$=perfect;
one of $M$ scores $1{-}1/M$).  We also report standard ten-bin
concept-marginal ECE~\cite{Guo2017Calibration} and the latent-atom
NLL.  Arrows in each table mark the improving direction.

\paragraph{MNIST-\SROIQ\ (two supervision regimes).}
Each example pairs two real MNIST images $a, b$ whose digits, restricted
to $0\!-\!4$, satisfy $\sym{succ}(a,b)$ with $b = (a{+}1)\bmod 5$ (a
restricted-digit task after rsbench~\cite{Bortolotti2024rsbench}).  One
per-image CNN runs on both slots; a \SROIQ\ ontology couples the pair
through \emph{every} distinctive \SROIQ\ feature: universal and
qualified-cardinality restrictions, the inverse and transitive
$\sym{succ}$ with the chain
$\sym{succ}\circ\sym{succ}\sqsubseteq\sym{plusTwo}$ (compiled at
$|\Delta|=3$), complement, covering disjunctions, and the functional,
(a)symmetric, and (ir)reflexive role characteristics;
App.~\ref{apd:exp} lists it in full.  The evidence
mixes a \emph{role} assertion ($\sym{succ}(a,b)$) with \emph{concept}
assertions ($\sym{Number}(a),\sym{Number}(b)$ and parity/primality);
digit identities are never supervised.  The amount of revealed concept
evidence controls $\#\mathrm{RS}$, the number of $\Theta$-consistent
completions the evidence leaves open (the modes of $K_e$,
\S\ref{sec:learn}; Table~\ref{tab:mnist}).

\emph{Grounded} ($\sym{succ}$, $\sym{Number}$, and a random half of the
parity/primality atoms; $\#\mathrm{RS}=1$): the evidence pins the digits
through the ontology, and WMC recovers the never-supervised identities
($0.25\!\to\!0.99$) at violation $0.02$ and concept expected
calibration error (ECE) $0.002$; the ontology-blind independent
perception stays at chance.

\emph{Under-determined} ($\sym{succ}$ and $\sym{Number}$ only;
$\#\mathrm{RS}=5$): the universal/inverse/functional axioms admit the
five cyclic relabelings of $0\!-\!4$.  The objective is invariant
under them, so a factorized perception cannot pick one even when
trained through the circuit: WMC reaches digit accuracy only $0.18$
(chance $0.20$) and leaves most decoded pairs successor-inconsistent
(violation $0.78$, concept ECE $0.15$); the shortcut is intrinsic to
the symmetry, not an optimization artifact.  Only \emph{grounding}
breaks it: a single parity atom drops $\#\mathrm{RS}$ to $3$ and the
full parity profile to $1$, after which WMC recovers the digits and
drives the violation to $0.02$ (the grounded row).

\begin{table}[t]
\floatconts
{tab:mnist}
{\caption{MNIST-\SROIQ, mean$\pm$std over ten seeds (metrics in
\S\ref{sec:exp}; \emph{digit} identities never supervised).  Grounding
the concept evidence collapses $\#\mathrm{RS}$ from $5$ to $1$.}}
{\small
\begin{tabular}{llcccc}
\toprule
regime & method & digit & latent & viol $\downarrow$ & ECE $\downarrow$ \\
\midrule
\multirow{2}{*}{grounded ($\#\mathrm{RS}{=}1$)} & Independent & $0.25${\scriptsize$\pm.09$} & $0.49${\scriptsize$\pm.14$} & $0.80${\scriptsize$\pm.12$} & $0.33${\scriptsize$\pm.14$} \\
 & WMC & $\mathbf{0.99}${\scriptsize$\pm.00$} & $\mathbf{1.00}${\scriptsize$\pm.00$} & $\mathbf{0.02}${\scriptsize$\pm.01$} & $\mathbf{0.00}${\scriptsize$\pm.00$} \\
\midrule
\multirow{2}{*}{under-det.\ ($\#\mathrm{RS}{=}5$)} & Independent & $0.20${\scriptsize$\pm.03$} & $0.46${\scriptsize$\pm.10$} & $0.93${\scriptsize$\pm.10$} & $0.08${\scriptsize$\pm.08$} \\
 & WMC & $0.18${\scriptsize$\pm.07$} & $0.67${\scriptsize$\pm.03$} & $0.78${\scriptsize$\pm.05$} & $0.15${\scriptsize$\pm.02$} \\
\bottomrule
\end{tabular}}
\end{table}

\paragraph{Pizzaiolo (a real role-based OWL ontology).}
Each example is a synthetically rendered pizza image from the Pizzaiolo
dataset~\cite{Bourguin2023Pizzaiolo}, encoded by a frozen ImageNet
ResNet-18 with a trainable linear head.  The OWL ontology shipped with the dataset is genuinely role-based
\SROIQ: a $\sym{hasTopping}$ role links each pizza to its toppings and the
pizza classes are defined by existential and universal restrictions over
it ($90$ existential and $44$ universal restrictions, $66$ named
classes).  It \emph{does not compile} to an
SDD: it grounds in under a second, but the SDD exhausts memory even for a
single pizza, a real-ontology instance of the circuit-size wall
(\S\ref{sec:limit}; quantified in App.~\ref{apd:size}, where grounding
stays sub-second while the SDD exceeds $22$\,GB by three toppings).

To still drive the ResNet through the role structure we extract a
compilable two-topping fragment ($59$ atoms): a $\sym{hasTopping}$ role
to named topping individuals with the shipped definitions
$\sym{NonVeg}\equiv\sym{Pizza}\sqcap\exists\sym{hasTopping}.\sym{Meat}$,
$\sym{Spicy}\equiv\sym{Pizza}\sqcap\exists\sym{hasTopping}.\sym{Spicy}$,
and $\sym{Veg}\equiv\sym{Pizza}\sqcap\neg\sym{NonVeg}$.  Compiling even
this needs three sound, opt-in grounding refinements, documented and
mechanized in App.~\ref{apd:hooks} (the rest of the paper's circuits
stay byte-identical): role atoms are pruned to $\sym{hasTopping}$'s
declared domain and range; the cubic equality theory is skipped; and a
closed qualified existential is materialized as the exact biconditional
$X(a)\Leftrightarrow\bigsqcup_t\sym{hasTopping}(a,t)\sqcap C(t)$ over the
named toppings, recovering a class from the \emph{absence} of a
topping, not only its presence.  The ResNet
predicts the toppings (supervised) while the property classes are latent:
WMC recovers them (property $0.45\!\to\!0.92$) and drives the violation
from $0.89$ to $0.06$, where the independent perception stays at chance
(Table~\ref{tab:pizza}, App.~\ref{apd:exp}).

The closed qualified-existential refinement is the one departure from
open-world semantics in this fragment, so we ablate it directly
(Table~\ref{tab:pizza-ablation}, App.~\ref{apd:hooks}).  Keeping the forward Skolemized
direction but dropping the closure clause
$\exists\sym{hasTopping}.C\!\to\!X(a)$ leaves the supervised toppings
untouched ($0.93\!\to\!0.95$) but lowers latent property recovery from
$0.92$ to $0.78$, which stays $0.33$ above the independent baseline's
$0.45$, so perception does not collapse; it becomes
\emph{over-permissive}.
The reported violation falls to $0.00$ because the removed
clause is the constraint whose violation was counted: the open model
satisfies a weaker theory vacuously rather than recovering the latents
it can no longer identify.  The closure direction therefore contributes
a $14$-point gain in latent recovery, which isolates it as a learning
signal over and above the EL-style forward grounding.  Without the
closure clause, the $\SROIQ$ fragment retains partial but degraded
identifiability.

\subsection{MNIST-Disjunction (RS-genuine, real images)}\label{sec:exp:disjunction}

The MNIST regimes above are \emph{over-determining} once a digit is
grounded, so single-WMC and JustWMC coincide.  We now build a task whose
latents admit several equally plausible $\Theta$-consistent completions
\emph{per example}, on real images.  Three individuals $a, b, c$ each
carry a real MNIST digit; $a, b$ show consecutive digits coupled by
$\sym{succ}(a,b)$, and a shared CNN recovers them through a \emph{digit}
circuit (the grounded MNIST mechanism).  On the same individuals an
image-free \emph{gender} attribute carries the reasoning shortcut: under
$\sym{Person}\sub\sym{Male}\sqcup\sym{Female}$ (disjoint),
$\sym{Male}\sub\forall\sym{marriedTo}.\sym{Female}$ with
$\sym{marriedTo}(a,b)$, and the parenthood roles
(App.~\ref{apd:exp}), exactly four $\Theta$-consistent gender modes
remain ($a,b$ opposite sex, $c$ free): $\#\mathrm{RS}=4$.  The pixels fix the digits
but \emph{not} the genders, so the gender modes are uninformed and
uniform; the Bayes-optimal latent NLL over the gender atoms is
$6\log 2\approx 4.16$ and the target is the calibrated posterior over
the four modes (RS diagnostics
after~\cite{Bortolotti2024rsbench,Marconato2024BEARS}).

Compiling both as one SDD is intractable (the two disjunctive
subsystems multiply out, \S\ref{sec:limit}), so we compile two
circuits over the same individuals -- a $114$-atom digit and a
$69$-atom gender circuit -- driven by one shared CNN with image-free
gender heads.  The split only separates the two independent
subsystems; the reasoning shortcut lives entirely within the
jointly-compiled gender circuit, so the target posterior is unaffected.

Table~\ref{tab:disjunction} (App.~\ref{apd:exp}) summarizes the outcome.  The CNN recovers
the digits through the circuit on every method (accuracy $0.99$).  For
the gender RS, no factorized method covers the four modes.
\textbf{Independent} hedges (NLL $6.93$) but leaks mass everywhere (TV
$1.00$); \textbf{single-WMC} (Semantic
Loss~\cite{Xu2018SemanticLoss}, conditioned
DeepProbLog~\cite{Manhaeve2018DeepProbLog}) and the \textbf{learned
mixture} (JustWMC, the BEARS-ensemble class~\cite{Marconato2024BEARS})
lower the NLL only marginally ($6.13$ / $6.69$) and do not spread (TV
$0.99$): the mixture-WMC objective is minimized by one confident head,
so gradient descent collapses onto a seed-dependent mode.
Representability does not imply
learnability~\cite{vanKrieken2024independence}.
\textbf{Justification-anchored JustWMC} closes the gap
\emph{constructively}: seeding the four heads from the circuit's
enumerated completions and learning only
the selector (which converges to uniform) realizes the positive
direction of Theorem~\ref{thm:rsia} in practice: Bayes-optimal NLL
$4.17$, zero concept ECE, and mode-coverage TV $0.02$ with no seed
variance.  The compiled circuit supplies the multimodal joint that gradient
descent over a factorized hypothesis class does not find.  An explicit
ensemble mitigation is consistent with this boundary between hypothesis
classes.  BEARS~\cite{Marconato2024BEARS}, a $K{=}4$ deep ensemble,
spreads mass across the modes, reaching mode-coverage TV $0.44$ where
the single-mode methods sit at $0.99$--$1.00$.  It stays far from the
anchored posterior nonetheless (TV $0.44$ against $0.02$, ECE $0.41$
against $0.00$, NLL $6.36$ against $4.17$) and varies across seeds
($\pm.11$ TV, $\pm2.36$ NLL).  A learned mixture therefore does not
recover the multimodal joint that anchoring reads directly from the
circuit.

\subsection{Size comparison with DeepProbLog}\label{sec:exp:dpl}

DeepProbLog (DPL)~\cite{Manhaeve2018DeepProbLog} compiles a
probabilistic logic program to an sd-DNNF, the same arithmetic-circuit
class our system builds on Horn inputs.  On the MNIST-\SROIQ\ ontology,
\textbf{12 of 52 axioms ($23\%$) leave ProbLog's Horn fragment} and need
hand-written constraint encodings; none is strictly inexpressible, but
each abandons the Horn structure DPL compiles natively
(Table~\ref{tab:dpl-comparison}, App.~\ref{apd:exp}).

\section{Related work}\label{sec:related}

The closest concurrent work is the SDD compiler
of~\cite{Lazzari2026Compiling}, which also trains a NeSy model through
a circuit compiled from a DL ontology; we extend the recipe to
\SROIQ, drive compilation by a consequence-based
calculus~\cite{TenaCucalaCuencaGrauHorrocks2018,TenaCucala2021AIJ}
rather than synthesizing the circuit from the semantics, and add a
mechanized metatheory.  That line addresses neither the reasoning
shortcuts we characterize and mitigate (\S\ref{sec:learn}) nor real
perception (it reports on hand-built feature vectors);
\cite{HerronJimenezRuizWeyde2023,HerronJimenezRuizWeyde2025} survey
the broader OWL-aware NeSy landscape.

The DL embeddings of \S\ref{sec:intro}, the lattice-saturated
embeddings of~\cite{Zhapa2025Lattice}, and fuzzy-DL
variants~\cite{Wu2022DFALC,Zhao2025DFELpp} optimize a relaxation of
OWL semantics on a continuous representation; they scale but compute no
classical WMC.  Embed2Sym~\cite{Aspis2022Embed2Sym} clusters embeddings
to recover symbols.  Within knowledge-compilation NeSy,
DeepProbLog~\cite{Manhaeve2018DeepProbLog} compiles to sd-DNNF; Semantic
Loss~\cite{Xu2018SemanticLoss} and Semantic Probabilistic
Layers~\cite{Ahmed2022SPL} train against a propositional WMC;
A-NeSI~\cite{vanKrieken2023ANeSI} amortizes it;
NeuPSL~\cite{Pryor2023NeuPSL}, Scallop~\cite{Li2023Scallop},
DeepStochLog~\cite{Winters2022DeepStochLog},
NeurASP~\cite{Yang2020NeurASP}, and
NeSyDM~\cite{vanKrieken2025NeSyDM} target propositional, Datalog, or
answer-set programs; we extend the lineage to OWL~2~DL.

Reductions of expressive DLs to rule languages predate NeSy (KAON2:
\ensuremath{\mathcal{SHIQ}} to disjunctive
Datalog~\cite{HustadtMotikSattler2007}); a reduction alone yields no
learner, but \method{} adds the path to an exact, differentiable WMC.

The independence assumption~\cite{Marconato2023RS,Marconato2024BEARS,vanKrieken2024independence,vanKrieken2025RSIA}
is the limit we attack; our experiments give a \SROIQ\ instance of the
BEARS recipe with an RS-awareness characterization
(Theorem~\ref{thm:rsia}).  Other probabilistic-DL systems differ in
target: DISPONTE~\cite{Riguzzi2015DISPONTE} fixes the world-sum;
BUNDLE~\cite{Riguzzi2013BUNDLE} and TRILL~\cite{Zese2019TRILL}
enumerate tableau explanations rather than one fixed circuit;
LNNs~\cite{Riegel2020LNN} use user-set rule weights.

\section{Limitations and future work}\label{sec:limit}

The guarded saturation that drives compilation (App.~\ref{apd:sat})
is sound but bounded, so it misses consequences of a deep ungrounded
existential chain.  Closing that gap calls for the full disjunctive
context calculus~\cite{TenaCucalaCuencaGrauHorrocks2018,TenaCucala2021AIJ},
sound and complete under the trivial expansion strategy but not yet
convergent on large ontologies; we treat it as a separate development.  Our benchmarks sidestep the
gap with grounding-covered, function-free constructs
(Theorem~\ref{thm:invariance}).  Multi-individual ABoxes inflate the
ground theory by $O(|\Delta|^2)$ in role atoms and $O(|\Delta|^3)$ in
the equality theory; larger ABoxes will need pairwise blocking, a sampling grounder, or
amortized
inference~\cite{vanKrieken2023ANeSI,vanKrieken2025NeSyDM}. One of the three Pizzaiolo grounding refinements, the closed-world reading of the \sym{NonVegetarianPizza} and \sym{SpicyPizza} qualified existentials, is a genuine departure from the open-world Tarskian semantics of \SROIQ: it materializes the biconditional $X(a) \Leftrightarrow \exists R.C$ so a latent class can be inferred from the absence of a filler, which the verified grounding (Theorem~\ref{thm:ground}) does not license. It is opt-in and confined to declared families; App.~\ref{apd:hooks} ablates it, and latent recovery degrades from 0.92 to 0.78 without it (the model becomes over-permissive rather than failing), so the refinement trades open-world faithfulness for identifiability under partial supervision.
Datatypes are out of scope.  The DeepProbLog comparison
(\S\ref{sec:exp:dpl}) counts the axioms whose encoding leaves ProbLog's
Horn fragment ($23\%$, an expressivity count, not an accuracy claim);
it does not run DeepProbLog end to end.  Our Independent, single-WMC,
and learned-mixture rows are controlled ablations of one pipeline: no
external NeSy system compiles a non-Horn ontology for a head-to-head
run.  In sum, the method's scope rests on
five assumptions, each stated where used: a finite active domain;
exactness only on grounding-covered or function-free existentials;
bounded, sound-but-incomplete saturation; the opt-in Pizzaiolo
grounding refinements; and, for anchored JustWMC, a $\#\mathrm{RS}$
small enough to enumerate.  Scaling
the perception to many-individual scenes is the natural next step, as is
evaluating anchored JustWMC beyond the enumerable-$\#\mathrm{RS}$
regime, on a reasoning-shortcut benchmark whose intended concepts are
identifiable from the perceptual input under limited concept
supervision, against an external mitigation baseline.

\section{Conclusion}\label{sec:conclusion}

\SROIQ\ admits neuro-symbolic learning by knowledge compilation:
consequence-based saturation over \ALCHOQ\ plus ABox grounding of the
\SROIQ\ features produce an SDD whose evidence-conditioned WMC trains
perception under partial supervision.  Where supervision
underdetermines the latents, anchoring the mixture to the circuit's
\emph{enumerated justifications} recovers the Bayes-optimal posterior
that learned mixtures miss~\cite{vanKrieken2025RSIA}.

\bibliographystyle{splncs04}
\bibliography{baobab}

\appendix

\section{\SROIQ\ syntax and semantics}\label{apd:semantics}

The body (\S\ref{sec:bg}) summarizes the fragment; we give the full
definition here for completeness.  A \SROIQ\ knowledge base is a triple
$(\mathcal{T}, \mathcal{R}, \mathcal{A})$ over a signature of atomic
concept names $\mathsf{N_C}$, atomic role names $\mathsf{N_R}$, and
individual names $\mathsf{N_I}$.

A \emph{role} is an atomic role $r \in \mathsf{N_R}$, an inverse role
$r^-$, or the universal role $U$.  \emph{Concepts} are generated by
\[
  C, D \;::=\; A \mid \top \mid \bot \mid \{a\} \mid \neg C \mid
  C \sqcap D \mid C \sqcup D \mid \exists R.C \mid \forall R.C \mid
  {\geq}\, n\, R.C \mid {\leq}\, n\, R.C \mid \exists R.\sym{Self},
\]
where $A \in \mathsf{N_C}$, $a \in \mathsf{N_I}$, $R$ a role, and
$n \in \mathbb{N}$.  In the implementation
(\texttt{baobab/sroiq/syntax.py}) these are the dataclasses
\texttt{ConceptName}, \texttt{Top}, \texttt{Bottom}, \texttt{Nominal},
\texttt{Not}, \texttt{And}, \texttt{Or}, \texttt{Exists},
\texttt{Forall}, \texttt{AtLeast}, \texttt{AtMost}, \texttt{HasSelf};
roles are \texttt{RoleName}, \texttt{InverseRole}, \texttt{UniversalRole}.
\texttt{And}/\texttt{Or} hold a \emph{flattened, deduplicated} frozen
set of operands.

Axioms are organized into three boxes.
The \emph{TBox} $\mathcal{T}$ is a set of general concept inclusions
$C \sub D$ (with $C \equiv D$ abbreviating the two inclusions and
$\mathsf{Disjoint}(C,D)$ abbreviating $C \sqcap D \sub \bot$).  The
\emph{ABox} $\mathcal{A}$ is a set of concept assertions $C(a)$, role
assertions $R(a,b)$, and (in)equalities $a \approx b$,
$a \not\approx b$.  The \emph{RBox} $\mathcal{R}$ contains role
inclusions $R \sub S$, role chains
$R_1 \circ \dots \circ R_k \sub S$, and the role characteristics
$\sym{Trans}(R)$, $\sym{Sym}(R)$, $\sym{Asym}(R)$, $\sym{Refl}(R)$,
$\sym{Irrefl}(R)$, $\sym{Func}(R)$, $\sym{InvFunc}(R)$,
$\sym{Disj}(R,S)$.  (As usual for decidability, \SROIQ\ requires the
RBox to be regular and number restrictions to use only simple roles;
our benchmarks satisfy both.)

The semantics is Tarskian.
An interpretation $\mathcal{I} = (\Delta^{\mathcal{I}}, \cdot^{\mathcal{I}})$
has a non-empty domain $\Delta^{\mathcal{I}}$ and maps each
$A \in \mathsf{N_C}$ to $A^{\mathcal{I}} \subseteq \Delta^{\mathcal{I}}$,
each $r \in \mathsf{N_R}$ to
$r^{\mathcal{I}} \subseteq \Delta^{\mathcal{I}} \times \Delta^{\mathcal{I}}$,
and each $a \in \mathsf{N_I}$ to
$a^{\mathcal{I}} \in \Delta^{\mathcal{I}}$.  Roles extend by
$(r^-)^{\mathcal{I}} = \{(y,x) : (x,y) \in r^{\mathcal{I}}\}$ and
$U^{\mathcal{I}} = \Delta^{\mathcal{I}} \times \Delta^{\mathcal{I}}$,
and concepts by
\begin{align*}
  \top^{\mathcal{I}} &= \Delta^{\mathcal{I}}, &
  \bot^{\mathcal{I}} &= \emptyset, &
  \{a\}^{\mathcal{I}} &= \{a^{\mathcal{I}}\}, \\
  (\neg C)^{\mathcal{I}} &= \Delta^{\mathcal{I}} \setminus C^{\mathcal{I}}, &
  (C \sqcap D)^{\mathcal{I}} &= C^{\mathcal{I}} \cap D^{\mathcal{I}}, &
  (C \sqcup D)^{\mathcal{I}} &= C^{\mathcal{I}} \cup D^{\mathcal{I}}.
\end{align*}
\begin{align*}
  (\exists R.C)^{\mathcal{I}} &= \{x : \exists y.\, (x,y) \in R^{\mathcal{I}} \wedge y \in C^{\mathcal{I}}\}, \\
  (\forall R.C)^{\mathcal{I}} &= \{x : \forall y.\, (x,y) \in R^{\mathcal{I}} \to y \in C^{\mathcal{I}}\}, \\
  ({\geq}\,n\,R.C)^{\mathcal{I}} &= \{x : \#\{y : (x,y) \in R^{\mathcal{I}} \wedge y \in C^{\mathcal{I}}\} \ge n\}, \\
  (\exists R.\sym{Self})^{\mathcal{I}} &= \{x : (x,x) \in R^{\mathcal{I}}\},
\end{align*}
with ${\leq}\,n\,R.C$ dual to ${\geq}\,(n{+}1)\,R.C$ under negation.
$\mathcal{I}$
satisfies $C \sub D$ iff $C^{\mathcal{I}} \subseteq D^{\mathcal{I}}$,
$R \sub S$ iff $R^{\mathcal{I}} \subseteq S^{\mathcal{I}}$,
$R_1 \circ \dots \circ R_k \sub S$ iff
$R_1^{\mathcal{I}} \circ \dots \circ R_k^{\mathcal{I}} \subseteq S^{\mathcal{I}}$,
and the role characteristics under their standard relational readings
($\sym{Trans}$: transitive; $\sym{Func}(R)$: $R^{\mathcal{I}}$ is
right-unique; $\sym{Disj}(R,S)$:
$R^{\mathcal{I}} \cap S^{\mathcal{I}} = \emptyset$; etc.).  This is the
semantics our Lean development takes as the definition of
truth (\S\ref{apd:lean}).

\section{Normalization and clausification}\label{apd:norm}

The compiler (\texttt{baobab/sroiq/normalisation.py}) maps each axiom to
\emph{DL-clauses} $\bigwedge_i B_i \to \bigvee_j H_j$ over the term
language of variables $x, y$, individual constants $a$, auxiliary
constants $o_\rho$ (introduced at grounding time for nominal /
cardinality witnesses), and unary Skolem terms $f(t)$.  Atoms are
concept atoms $C(t)$, role atoms $R(t_1, t_2)$, and equalities
$t_1 \approx t_2$ (\texttt{dlclauses.py}).

The first step is negation-normal form.
\texttt{nnf} pushes negation to atomic concepts and nominals using the
de~Morgan, quantifier, and cardinality dualities:
{\small
\[
\begin{array}{lll}
  \neg\top \rightsquigarrow \bot, &
  \neg\bot \rightsquigarrow \top, &
  \neg\neg C \rightsquigarrow C, \\
  \neg(C \sqcap D) \rightsquigarrow \neg C \sqcup \neg D, &
  \neg(C \sqcup D) \rightsquigarrow \neg C \sqcap \neg D, &
  \neg\exists R.C \rightsquigarrow \forall R.\neg C, \\
  \neg\forall R.C \rightsquigarrow \exists R.\neg C, &
  \neg({\geq}\,n\,R.C) \rightsquigarrow {\leq}\,(n{-}1)\,R.C, &
  \neg({\leq}\,n\,R.C) \rightsquigarrow {\geq}\,(n{+}1)\,R.C, \\
  \multicolumn{3}{l}{\neg({\geq}\,0\,R.C) \rightsquigarrow \bot,}
\end{array}
\]}
$\neg A$ and $\neg\{a\}$ remain (atomic literals);
$\neg\exists R.\sym{Self}$ is preserved for grounding.

The structural transformation follows.
Each non-atomic subconcept $D$ is assigned a fresh concept name $Q_D$
(method \texttt{q}) and defined by clauses for $Q_D \equiv D$
(method \texttt{\_define}); a \SROIQ\ inclusion $C \sub D$ then becomes
the single clause $Q_C(x) \to Q_D(x)$.  The definitional clauses are:
\begin{itemize}\itemsep2pt
  \item $Q \equiv C \sqcap D$:\quad $Q(x) \to Q_C(x)$,\;
        $Q(x) \to Q_D(x)$,\; $Q_C(x) \wedge Q_D(x) \to Q(x)$.
  \item $Q \equiv C \sqcup D$:\quad $Q(x) \to Q_C(x) \vee Q_D(x)$,\;
        $Q_C(x) \to Q(x)$,\; $Q_D(x) \to Q(x)$.
  \item $Q \equiv \exists R.C$:\quad $Q(x) \to R(x, f_Q(x))$,\;
        $Q(x) \to Q_C(f_Q(x))$,\; $R(x,y) \wedge Q_C(y) \to Q(x)$,
        where $f_Q$ is a unary Skolem function.
  \item $Q \equiv \forall R.C$:\quad $Q(x) \wedge R(x,y) \to Q_C(y)$.
  \item $Q \equiv \neg A$:\quad $Q(x) \wedge A(x) \to \bot$,\;
        $\to Q(x) \vee A(x)$.
  \item $Q \equiv \exists R.\sym{Self}$:\quad $Q(x) \to R(x,x)$,\;
        $R(x,x) \to Q(x)$.
  \item $Q \equiv {\geq}\,n\,R.C$:\quad for $i = 0, \dots, n{-}1$,
        $Q(x) \to R(x, f_{Q,i}(x))$ and $Q(x) \to Q_C(f_{Q,i}(x))$;
        pairwise distinctness
        $Q(x) \wedge f_{Q,i}(x) \approx f_{Q,j}(x) \to \bot$ for
        $i < j$; closure $R(x,y) \wedge Q_C(y) \to Q(x)$ when $n = 1$.
  \item $Q \equiv {\leq}\,n\,R.C$:\quad
        $Q(x) \wedge \bigwedge_{t=0}^{n} \big(R(x, y_t) \wedge Q_C(y_t)\big)
         \to \bigvee_{i < j} y_i \approx y_j$ \;($n{+}1$ fresh successor
        variables).
\end{itemize}
Nominals $\{a\}$ are not given definitional clauses: they map to a
stable proxy $\sym{Nom}_a$ and an entry
$\sym{Nom}_a \mapsto a$ in the \texttt{GroundHooks} record.  Qualified
number restrictions additionally register a tuple $(Q, n, R, Q_C)$ in
\texttt{hooks.at\_least} / \texttt{hooks.at\_most}; inverse roles
register $(R, S)$ in \texttt{hooks.role\_inverses} and resolve to a
proxy role name; the universal role emits the fact $\to U(x,y)$.  RBox
characteristics are routed either to native clauses (symmetric:
$R(x,y) \to R(y,x)$; reflexive: $\to R(x,x)$; irreflexive:
$R(x,x) \to \bot$; functional: $R(x,y_0) \wedge R(x,y_1) \to y_0 \approx
y_1$; inverse-functional dually) or to the grounding record
(transitive, chains, asymmetric, disjoint, nominal).

\section{Consequence-based saturation}\label{apd:sat}

The DL-clause set is saturated under propositional
\emph{hyperresolution with most-general-unifier matching}
(\texttt{baobab/sroiq/cb\_saturation.py}): from $\Gamma \to \Delta \vee A$
and $A' \wedge \Gamma' \to \Delta'$ with $A\sigma = A'\sigma$ the
MGU resolvent $(\Gamma \wedge \Gamma')\sigma \to (\Delta \vee
\Delta')\sigma$ is added, unification using Robinson's algorithm with
occurs-check over the four term kinds.  A clause is admitted to the
pool only if it passes four guards that bound the otherwise unbounded
context machinery and force termination:
\begin{enumerate}\itemsep2pt
  \item \textbf{tautology elimination}: drop any clause whose head and
        body share an atom;
  \item \textbf{variable bound}: at most
        $\max(2, \textit{max-vars-in-input})$ distinct variables per
        clause (two suffices for \ALCHOQ; transitivity and two-step
        chains need three; ${\leq}\,n$ needs $n{+}2$);
  \item \textbf{role-atom bound}: at most
        $\max(2, \textit{max-role-atoms-in-input})$ role atoms per body
        (a length-$k$ chain needs $k$; ${\leq}\,n$ needs $n{+}1$);
  \item \textbf{term-depth bound}: every function term has nesting
        depth $\le 1$, blocking unbounded Skolem chains $f(g(\dots))$.
\end{enumerate}
The fixpoint loop re-resolves all clause pairs until a full pass adds
nothing.

\begin{lemma}[Termination and soundness of the guarded saturation]\label{lem:sat}
The guarded saturation terminates, and every clause it derives is a
logical consequence of its input.
\end{lemma}
\begin{proof}
Termination: the four guards keep the reachable atom-and-clause
universe finite (bounded variables, bounded role atoms, depth-$\le 1$
terms over a finite signature) and tautologies are removed, so the
subsumption-minimal pool is finite and the fixpoint loop halts.
Soundness: each resolution step adds only entailed clauses; the
property is mechanized for the underlying \ALCHOQ\ saturation
(axiom-free; \S\ref{apd:lean}).
\end{proof}

The calculus forgoes \emph{completeness}: the bounded calculus need not
derive every subsumption that the full \ALCHOIQ\ context
calculus~\cite{TenaCucalaCuencaGrauHorrocks2018,TenaCucala2021AIJ}
would.  Incompleteness is harmless for the consequences a ground-time
grounding already discharges (nominals, number restrictions, local
reflexivity, role characteristics): there the grounded CNF
$\Theta_\Delta$ is faithful on its own (\S\ref{apd:hooks}), so any
clause saturation adds is already entailed by $\Theta_\Delta$ and
changes neither the model set nor, over a fixed vtree, the circuit
(\S\ref{apd:size}).  It is \emph{not} harmless for existential
restrictions grounding does not cover: their consequences enter
$\Theta_\Delta$ only through saturation, so an incomplete or skipped
saturation yields a sound but over-permissive circuit that admits
models the ontology forbids.  The trade-off is deliberate: we bound the
calculus for termination and accept incompleteness on deep existential
chains.  On a grounding-covered, function-free fragment saturation adds
nothing (Theorem~\ref{thm:invariance}), so we skip it and ground
directly (\S\ref{sec:limit}).

\section{Grounding the \SROIQ\ features}\label{apd:hooks}

At grounding time the \texttt{GroundHooks} record is materialized over
the active domain $\Delta$ ($d = |\Delta|$;
\texttt{baobab/grounding/}).  Each grounding rule emits a bounded number of
propositional clauses, and each is justified by a Tarskian-semantics
soundness lemma (\S\ref{apd:lean}).
Table~\ref{tab:hooks} lists every \texttt{GroundHooks} rule with its emitted
clause schema and ground-clause count; the three optional refinements of
\S\ref{sec:exp:images} are documented separately below; the closed qualified-existential refinement is ablated in Table~\ref{tab:pizza-ablation}.

\begin{table}[t]
\floatconts
  {tab:hooks}
  {\caption{Grounding schemata over the active domain $\Delta$,
   $d = |\Delta|$.  $a, b, c$ range over $\Delta$; $\binom{d}{n}$ is the
   $n$-subset count.}}
  {\footnotesize
   \begin{tabular}{lll}
   \toprule
   feature & emitted ground clauses & count \\
   \midrule
   nominal $\sym{Nom}_a$ & $\sym{Nom}_a(a)$;\; $\neg\sym{Nom}_a(b)$, $b \ne a$ & $d$ \\
   ${\geq}\,n\,R.C$ at $x$ & $Q(a) \leftrightarrow \bigvee_{|S|=n}\bigwedge_{y\in S} R(a,y)\wedge C(y)$ & $d\,(\binom{d}{n}{+}1)$ \\
   ${\leq}\,n\,R.C$ at $x$ & $Q(a) \leftrightarrow \bigwedge_{|T|=n+1}\big(\bigwedge_{y\in T} R(a,y)\wedge C(y) \to \bigvee_{i<j} y_i \approx y_j\big)$ & $O(d^{\,n+1})$ \\
   inverse $R \equiv S^-$ & $R(a,b) \leftrightarrow S(b,a)$ & $2d^2$ \\
   $\sym{Trans}(R)$ & $R(a,b) \wedge R(b,c) \to R(a,c)$ & $d^3{-}d$ \\
   chain $R_1{\circ}\dots{\circ}R_k \sub S$ & $R_1(a_0,a_1)\wedge\dots\wedge R_k(a_{k-1},a_k) \to S(a_0,a_k)$ & $d^{\,k+1}$ \\
   $\sym{Sym}(R)$ & $R(a,b) \to R(b,a)$ & $d^2{-}d$ \\
   $\sym{Asym}(R)$ & $\neg\big(R(a,b) \wedge R(b,a)\big)$, $a \ne b$ & $d^2{-}d$ \\
   $\sym{Disj}(R,S)$ & $\neg\big(R(a,b) \wedge S(a,b)\big)$ & $d^2$ \\
   $\sym{Refl}(R)$ & $R(a,a)$ & $d$ \\
   $\sym{Irrefl}(R)$ & $\neg R(a,a)$ & $d$ \\
   $\sym{Func}(R)$ & $R(a,b) \wedge R(a,c) \to b \approx c$ & $d^2{-}d$ \\
   $\sym{InvFunc}(R)$ & $R(b,a) \wedge R(c,a) \to b \approx c$ & $d^2{-}d$ \\
   $\exists R.\sym{Self}$ proxy $Q$ & $Q(a) \leftrightarrow R(a,a)$ & $2d$ \\
   universal $U$ & $U(a,b)$ & $d^2$ \\
   \bottomrule
   \end{tabular}}
\end{table}

The grounder also allocates a propositional atom for every concept
name applied to each individual ($O(|\mathsf{N_C}|\,d)$) and every role
applied to each ordered pair ($O(|\mathsf{N_R}|\,d^2)$); when the ontology
uses nominals, qualified number restrictions, or (inverse-)functional
roles it additionally emits an equality theory ($d^2$ atoms with
reflexivity, symmetry, transitivity, and concept/role congruence),
otherwise the equality predicate is omitted (refinement (ii) below).
The dominant terms are the cubic
transitivity / equality-transitivity clauses and, for a length-$k$
chain or a number restriction of bound $n$, the $d^{\,k+1}$ /
$O(d^{\,n+1})$ entries, the polynomial-degree bounds proved in
our Lean development (\S\ref{apd:lean}).

\paragraph{Three optional refinements.}  The role-typing, equality-skip,
and closed-existential refinements used for the Pizzaiolo fragment
(\S\ref{sec:exp:images}) are opt-in and off by default, so every other
circuit in the paper is byte-identical to the unrefined grounding.
Refinements (i) and (ii) are both instances of one fact about weighted
model counting: it factorizes over an independent partition of the atoms,
and a block whose sub-theory has a unique satisfying assignment contributes
a single constant factor.  We mechanize that fact and the two refinements
as model-count-preservation lemmas (\texttt{wmc\_split},
\texttt{wmc\_factor}, \texttt{wmc\_prune\_forcedFalse},
\texttt{wmc\_skip\_equality}, with exact count preservation
\texttt{wmc\_card\_preserved} and constant log-offset \texttt{wmc\_log\_offset};
\S\ref{apd:lean}).  Refinement (iii) is a closed-world reading justified
only by the Tarskian argument below.
\emph{(i) Domain/range typing.}  When a role $R$ declares a domain $D_R$
and range $E_R$ and the named individuals are typed accordingly, $R(a,b)$
is false in every model unless $a \in D_R$ and $b \in E_R$, so allocating
its atom only for those pairs replaces the $|\mathsf{N_R}|\,d^2$ role atoms
by $\sum_R |D_R|\,|E_R|$.  The pruned atoms form an independent block whose
only model sets them all false, so the model count is preserved exactly and
the WMC up to the constant $\prod \mathit{w}(\text{false})$.
\emph{(ii) Equality skip.}  On an ontology free of nominals, number
restrictions, and functional roles, no clause outside the equality theory
mentions $\approx$; the $\approx$ atoms then form an independent
sub-theory with, under the unique-name default, a single model, so dropping
it preserves the model count exactly and the weighted count up to a constant
(the training gradient and the consistency check are identical).
\emph{(iii) Closed existential.}  When a role's fillers are closed to the
named domain, a definition $X \equiv \exists R.C$ whose head Skolem witness
the grounder drops (\S\ref{sec:calc}) is recovered as the exact
biconditional $X(a) \leftrightarrow \bigvee_{b \in \Delta} R(a,b) \wedge
C(b)$ ($2d$ clauses), restoring the backward and closure directions
($X \to \exists$ and $\neg\exists \to \neg X$) that Skolemization omits.
Unlike (i) and (ii) this is a closed-world reading: it is sound only under
the named-filler closure assumption, which holds for the Pizzaiolo dataset
(each pizza's topping list is fully observed) but is not an
open-world consequence and is therefore outside the mechanized
guarantees.

\begin{table}[t]
\floatconts
{tab:pizza-ablation}
{\caption{Closed-world ablation on the Pizzaiolo fragment (mean$\pm$std
over ten seeds, five epochs; same fragment and metrics as
Table~\ref{tab:pizza}).  \emph{Closed} materializes the exact
biconditional $X(a)\Leftrightarrow\bigsqcup_t \sym{hasTopping}(a,t)\sqcap C(t)$;
\emph{open} keeps the forward direction $X(a)\!\to\!\exists\sym{hasTopping}.C$
only, dropping the closure clause $\exists\sym{hasTopping}.C\!\to\!X(a)$.
Removing the closure clause leaves supervised toppings untouched but
lowers latent property recovery from $0.92$ to $0.78$; the reported
violation falls to $0.00$ because the removed clause is the constraint
whose violation was counted (\S\ref{sec:exp:images}).}}
{\small
\begin{tabular}{llccc}
\toprule
refinement & method & topping & property & viol $\downarrow$ \\
\midrule
closed & Independent & $0.95${\scriptsize$\pm.01$} & $0.45${\scriptsize$\pm.15$} & $0.89${\scriptsize$\pm.18$} \\
closed & WMC         & $0.93${\scriptsize$\pm.01$} & $\mathbf{0.92}${\scriptsize$\pm.02$} & $0.06${\scriptsize$\pm.02$} \\
\midrule
open   & Independent & $0.95${\scriptsize$\pm.01$} & $0.45${\scriptsize$\pm.16$} & $0.86${\scriptsize$\pm.28$} \\
open   & WMC         & $0.95${\scriptsize$\pm.01$} & $0.78${\scriptsize$\pm.01$} & $0.00${\scriptsize$\pm.00$} \\
\bottomrule
\end{tabular}}
\end{table}

\section{SDD compilation, WMC, and the distribution semantics}\label{apd:wmc}

Algorithm~\ref{alg:compile} states the end-to-end workflow referenced in
\S\ref{sec:calc}: the five compile-time stages run once per ontology, and
the differentiable WMC pass runs once per training step.

\begin{algorithm}[t]
\caption{Compile a \SROIQ\ ontology to a differentiable circuit and
  train the perception network}
\label{alg:compile}
\KwIn{ontology $\Onto$, active domain $\Delta$, perception network
  $p_\theta$, supervised instances $\{(x^{(k)}, e^{(k)})\}$}
\tcc{Compile (once):}
$N \gets \mathrm{NNF}(\Onto)$\tcp*{normalization, App.~\ref{apd:norm}}
$\mathcal{C} \gets \mathrm{Clausify}(N)$\tcp*{structural transformation}
$\mathcal{C}^{*} \gets \mathrm{Saturate}(\mathcal{C})$\tcp*{hyperresolution; skip when grounding-covered}
$\Theta_\Delta \gets \mathrm{Ground}(\mathcal{C}^{*}, \Delta)$\tcp*{expand nominals, number restrictions, roles}
$S \gets \mathrm{CompileSDD}(\Theta_\Delta)$\tcp*{PySDD over a fixed vtree}
\tcc{Train (per step):}
\ForEach{instance $(x^{(k)}, e^{(k)})$}{
  $p \gets p_\theta(x^{(k)})$\tcp*{per-atom probabilities}
  $\ell \gets -\log \mathrm{WMC}(S, p, e^{(k)})$\tcp*{evidence-conditioned}
  backpropagate $\ell$ through $S$ into $\theta$\;
}
\end{algorithm}

The grounded CNF is compiled to an SDD with PySDD.  The differentiable
WMC layer (\texttt{baobab/nesy/wmc\_layer.py}, class \texttt{WmcLayer})
evaluates the circuit bottom-up: a positive literal for atom $a$
contributes its probability $p_a$, a negative literal $1 - p_a$, the
$\mathsf{true}$/$\mathsf{false}$ constants contribute $1$/$0$, and a
decision node with prime/sub partition $\{(p_i, s_i)\}$ contributes
$\sum_i \WMC(p_i)\cdot\WMC(s_i)$, memoized so each node is visited
once.  Probabilities are clamped to $[\varepsilon, 1{-}\varepsilon]$
for log-stability; the pass is exact in floating point and
differentiable by autodiff, so gradients flow into $p_\theta$.
Evidence conditioning (\texttt{wmc\_with\_evidence}) fixes $p_a = 1$
(resp.\ $0$) for each $(a, \mathsf{true}) \in e$ (resp.\ $\mathsf{false}$)
and counts on the resulting weights, yielding
\[
  \WMC(\Theta \wedge e \mid p_\theta)
  = \sum_{\substack{\alpha \,\models\, \Theta \\ \alpha \,\supseteq\, e}}
    \;\prod_{a} p_a^{\,\alpha_a}\,(1 - p_a)^{\,1-\alpha_a},
\]
the sum over satisfying assignments $\alpha$ of their product weight
(Theorem~\ref{thm:wmc} states this in full).

\begin{proof}[Proof of Theorem~\ref{thm:wmc}]
Reading each atom as an independent Bernoulli event with parameter
$p_a$ is exactly the distribution
semantics~\cite{Riguzzi2015DISPONTE}: the displayed expression is the
total probability mass of the interpretations that satisfy
$\Theta \wedge e$, the world-sum, and taking $e = \emptyset$ gives the
unconditioned statement of Theorem~\ref{thm:wmc}.  The bottom-up
evaluation above computes exactly this sum, and the correspondence on
the compiled circuit is mechanized in Lean (\S\ref{apd:lean}).
\end{proof}
The world-sum is the quantity the losses of \S\ref{apd:loss} optimize.

\section{Saturation and circuit size}\label{apd:size}

\S\ref{sec:calc} divides a construct's consequences into those a
grounding discharges and those only saturation can supply.  Here
we make that division precise on two axes: circuit \emph{size}, on which
saturation is provably inert, and the \emph{model set}, where
saturation is decisive for existentials that no grounding rule covers.  Recall that the
compiler (\S\ref{apd:wmc}) builds the SDD over a \emph{fixed} vtree,
conjoining the ground clauses with PySDD's \texttt{Apply}, which keeps
every node compressed and trimmed.

\begin{theorem}[Saturation is circuit-invariant under a fixed vtree]\label{thm:invariance}
Fix a vtree $v$.  Let $\Theta_\Delta$ be the grounded CNF and let
$\Theta_\Delta^{+}$ add to it any set of clauses entailed by
$\Theta_\Delta$ (as sound saturation does, Lemma~\ref{lem:sat}).  The
compressed, trimmed SDDs of $\Theta_\Delta$ and $\Theta_\Delta^{+}$ over
$v$ are identical; in particular they have the same node count and the
same weighted model count.
\end{theorem}
\begin{proof}
Saturation resolves over the existing signature and introduces no new
ground atoms, so $\Theta_\Delta^{+}$ is a CNF over the same variables as
$\Theta_\Delta$ and is compiled over the same vtree.  Sound saturation
adds only entailed clauses (Lemma~\ref{lem:sat}), so
$\Theta_\Delta^{+}\equiv\Theta_\Delta$ as Boolean functions.  For a
fixed vtree the compressed and trimmed SDD of a function is
unique~\cite{Darwiche2011SDD}, and \texttt{Apply} maintains compression
and trimming, so the compiler returns this canonical diagram regardless
of which equivalent clause set presents the function.  The two SDDs, and
hence their sizes and weighted model counts, therefore coincide.  (The
fixed vtree is essential: over a different variable order the same
function can compile to a different size, so the claim is about adding
clauses, not about reordering.)
\end{proof}

Consequently, under the fixed-vtree compilation we use, saturation is
\emph{exactly} size-neutral: it cannot shrink the circuit (nor enlarge
it).  A size reduction becomes available only once the compiler is
allowed to search for a vtree, where the compiled size is no longer
canonical across logically equivalent inputs.

\begin{proposition}[Under vtree minimization the reduction is real but not monotone]\label{prop:minimize}
With dynamic vtree minimization enabled, there are ontologies on which
saturation strictly decreases the minimized SDD size and ontologies on
which it strictly increases it.
\end{proposition}
\begin{proof}
By the witnesses in Table~\ref{tab:satsize}: for \texttt{disjunction} the
minimized size drops from $17$ to $12$ after saturation, while for
\texttt{bird-penguin} it rises from $139$ to $143$.
\end{proof}

\begin{table}[t]
\floatconts
  {tab:satsize}
  {\caption{Effect of saturation on SDD node count, measured with the
   accompanying library.  Saturation adds entailed clauses (column
   \emph{clauses}) without changing the atom set.  Over a \emph{fixed}
   vtree the size is identical with and without saturation
   (Theorem~\ref{thm:invariance}); under dynamic vtree
   \emph{minimization} it usually shrinks but is not guaranteed to
   (\texttt{bird-penguin}).}}
  {\footnotesize
   \begin{tabular}{lcccc}
   \toprule
                       & clauses        & fixed vtree & \multicolumn{2}{c}{minimized} \\
   \cmidrule(lr){4-5}
   ontology            & unsat$\to$sat  & unsat $=$ sat & unsat & sat \\
   \midrule
   horn-chain          & $3\to6$        & $10$  & $8$   & $8$ \\
   disjunction         & $7\to12$       & $23$  & $17$  & $\mathbf{12}$ \\
   two-disjunctions    & $11\to29$      & $38$  & $22$  & $\mathbf{18}$ \\
   exists+disjunction  & $8\to73$       & $73$  & $41$  & $\mathbf{36}$ \\
   bird-penguin        & $11\to44$      & $1011$& $139$ & $143$ \\
   \bottomrule
   \end{tabular}}
\end{table}

Every row of Table~\ref{tab:satsize} is an ontology on which grounding is
already faithful (the constructs are function-free or grounding-covered), so
saturation only re-derives clauses already entailed by $\Theta_\Delta$
and Theorem~\ref{thm:invariance} applies.  The other regime is where
saturation is \emph{not} redundant.

\begin{proposition}[Saturation is load-bearing for ungrounded existentials]\label{prop:loadbearing}
There are ontologies on which skipping saturation strictly enlarges the
model set and drops an entailment that holds in every model of the
ontology.
\end{proposition}
\begin{proof}
Take $A \sub \exists R.B$, $B \sub C$, $\exists R.C \sub D$ with $A(a)$
over $\Delta = \{a\}$; the ontology entails $D(a)$.  With saturation the
grounded theory has $6$ models and forces $D(a)$; without it the head
existential $A \sub \exists R.B$ clausifies to Skolem-term clauses the
grounder drops, the theory has $14$ models, and $D(a)$ holds in only a
$0.71$ fraction of them.  The two circuits compute different functions.
\end{proof}

The two facts delimit the role of entailment exactly.  Where grounding
is already faithful, saturation cannot change the circuit
(Theorem~\ref{thm:invariance}); where it is not, saturation is the only
route by which an ungrounded existential's consequences reach the finite
theory (Proposition~\ref{prop:loadbearing}).  Saturation's job is
\emph{completeness}, not compression.  Our benchmarks lie entirely in
the first regime: every construct they use is function-free or
grounding-covered, so their grounded circuits are faithful
(Theorem~\ref{thm:ground}) and saturation is provably unnecessary for
them; we therefore report results without relying on it, and without
any wall-clock cutoff.

The full workflow of this section, normalization, consequence-based
saturation, grounding, and SDD compilation, is implemented as an
installable Python package (\texttt{baobab/sroiq/} in the released
code), with the example ontologies used above under
\texttt{experiments/ontologies/}; the measurements in
Table~\ref{tab:satsize} were produced with it.

\begin{table}[t]
\floatconts
  {tab:pizza-growth}
  {\caption{Compiling the \emph{full} Pizza\"iolo \SROIQ\ ontology
   ($66$ named classes, $90$ existential and $44$ universal
   restrictions, right-hand-side disjunctions on \sym{Pizza} and
   \sym{SpicyPizza}) over a growing active domain, with a fixed vtree
   throughout.  Grounding stays sub-second at every size; the compiled
   SDD grows super-linearly and exceeds a $22$\,GB / $600$\,s budget at
   three toppings.  The $59$-atom two-topping fragment of
   \S\ref{sec:exp:images} is a separately extracted role-based
   sub-signature, not a row here.}}
  {\footnotesize
   \begin{tabular}{lccccc}
   \toprule
   toppings ($|\Delta|$) & atoms & clauses & ground (s) & SDD nodes & peak mem \\
   \midrule
   $1$ ($|\Delta|{=}2$) & $201$ & $638$       & $0.04$ & $3{,}182$  & $0.7$\,GB \\
   $2$ ($|\Delta|{=}3$) & $302$ & $957$       & $0.19$ & $12{,}064$ & $5.7$\,GB \\
   $3$ ($|\Delta|{=}4$) & $403$ & $1{,}276$   & $0.01$ & ---        & $>22$\,GB$^{\dagger}$ \\
   \bottomrule
   \end{tabular}\\[2pt]
   {\scriptsize $^{\dagger}$compilation exceeds the $600$\,s / $22$\,GB
   budget; grounding still completes in $0.01$\,s.}}
\end{table}

\paragraph{Diagnosing the Pizza\"iolo compilation wall.}
The full ontology of \S\ref{sec:exp:images} is the object measured in
Table~\ref{tab:pizza-growth}; the $59$-atom fragment we train on is a
deliberately minimized role-based sub-signature, so its size is not
comparable to the rows of that table.  The three candidate bottlenecks
separate.  \emph{Grounding size} is not the cause: the ground
theory grows linearly ($201\to302\to403$ atoms) and grounding completes
in under $0.2$\,s at every size, including the three-topping instance
that fails to compile ($0.01$\,s).  The \emph{vtree} is not the cause
either: compilation uses a single fixed vtree throughout
(Theorem~\ref{thm:invariance}), so the growth is a property of the
compiled function, not of the variable order.  The bottleneck is the
\emph{intrinsic circuit size}: the SDD grows super-linearly in the
grounding ($3{,}182\to12{,}064$ nodes, a $3.8\times$ jump for a
$1.5\times$ clause increase), compile time rises $66\times$
($0.9\to60$\,s) for one added topping, and peak memory exceeds $22$\,GB
at three toppings.  The mechanism is treewidth: the $67$ disjointness
axioms together with the right-hand-side disjunctions
($\sym{Pizza}\sqsubseteq\sym{VegetarianPizza}\sqcup\sym{NonVegetarianPizza}$,
$\sym{SpicyPizza}\equiv\bigsqcup_i \sym{Topping}_i$) and the
\sym{VegetarianPizza} complement couple the topping atoms across each
pizza, raising the primal-graph treewidth, and SDD size is worst-case
exponential in treewidth (\S\ref{sec:limit}).  Grounding size feeds this
growth but is not itself the wall, which is why a hand-pruned
sub-signature compiles where the full ontology cannot.

\section{Learning objectives}\label{apd:loss}

Let $p_\theta(\cdot \mid x)$ be the perception's per-atom posterior on
example $x$ and $e$ its revealed evidence
(\texttt{baobab/nesy/wmc\_layer.py} and the \texttt{experiments/}
drivers).
\begin{description}\itemsep3pt
  \item[Independent.] $\mathcal{L} = \frac{1}{|e|}\sum_{(a,v)\in e}
    \mathrm{BCE}(p_\theta(a), v)$, binary cross-entropy on the
    observed atoms only; the constraint is ignored.
  \item[WMC.] $\mathcal{L} = \mathrm{BCE}(p_\theta, v)\big|_e
    + \lambda\big({-}\log \WMC(\Theta \wedge e \mid p_\theta)\big)$
    (Eq.~\ref{eq:wmc-loss}); the second term is the
    Semantic-Loss/DeepProbLog objective on the compiled circuit.
  \item[JustWMC.] A single \texttt{MixtureEncoder} with a shared MLP
    body, $K$ atom heads $\mu_k$, and a softmax selector $\pi$,
    trained on
    \[
      \mathcal{L}_\text{Just} =
        \underbrace{\textstyle\sum_k \pi_k\,\mathrm{BCE}(\mu_k, v)\big|_e}_{\text{mixture label term}}
        + \lambda\Big({-}\log \textstyle\sum_k \pi_k\, \WMC(\Theta \wedge e \mid \mu_k)\Big)
        - \kappa\cdot\tfrac{1}{K}\textstyle\sum_k \mathrm{KL}\!\big(\mu_k \,\big\|\, \mathrm{sg}(\bar\mu)\big),
    \]
    where $\bar\mu = \frac1K\sum_{k} \mu_{k}$ and $\mathrm{sg}$ is
    stop-gradient (\texttt{.detach()}).  The first two terms are
    permutation-invariant in the heads, so the KL-from-mean diversity
    term is required to break head symmetry and let the heads occupy
    distinct $\Theta$-consistent modes; the test-time prediction is the
    mixture marginal $\bar\mu_\pi = \sum_k \pi_k \mu_k$.
  \item[BEARS.] $K$ separately trained UCI encoders
    (\cite{Marconato2024BEARS}; \texttt{train\_bears}); member $k$
    minimizes its semantic loss plus a KL to the running average of
    members $0..k$ and a Bernoulli-entropy term, with members $0..k{-}1$
    frozen.  We report both the ensemble mean and the best-of-$K$
    oracle.
\end{description}
Read as ablations, the baselines isolate each mechanism: the gap of
\textbf{WMC} over \textbf{Independent} is the contribution of weighted
model counting over the compiled $\Theta$ (circuit present vs.\ absent),
while the gap of \textbf{anchored} over \textbf{JustWMC} is the
contribution of the SDD-compiled support over a learned mixture support
(support compiled vs.\ learned).
The JustWMC construction is the empirical face of a result we verify
in Lean (\S\ref{apd:lean}): every reasoning-shortcut mixture over a
\SROIQ\ query is representable by a categorical mixture of UCIs, whose
components are indexed by the query's justifications~\cite{vanKrieken2025RSIA}.

\section{Experimental details}\label{apd:exp}

All drivers compile the SDD once, then train each method with Adam;
every reported metric is computed on a held-out evaluation split of
\emph{unseen} individuals (fresh image pairs and pizzas) disjoint from
the training examples, so the numbers measure generalization to new
query individuals rather than memorization (sizes per experiment below).
The MNIST drivers score real $28{\times}28$ images with a CNN
(learning rate $10^{-3}$, batch $128$); the Pizzaiolo fragment trains a
linear head on a frozen ResNet-18 (learning rate $10^{-3}$, batch
$64$); the synthetic Pizza control reads a $12$-bit feature vector with
per-bit flip noise $0.05$ through a small MLP (learning rate $10^{-2}$,
batch $32$).  Table~\ref{tab:hparams} collects the hyperparameters; the
headline numbers below are the committed reference runs
(\texttt{results/} in the code repository) reported in \S\ref{sec:exp}.

All reference runs were produced on a single NVIDIA GeForce RTX~4090
(24\,GB) under Slurm; no experiment needs more than one GPU.
Compilation is cheap relative to training: the $114$-atom
MNIST-\SROIQ\ circuit compiles to an SDD of $(7.3$--$9.7){\cdot}10^3$
nodes in $1$--$3$\,s on CPU (the spread across runs comes from PySDD's
dynamic vtree minimization), the MNIST-Disjunction pair to $10{,}111$
(digit) and $2{,}972$ (gender) nodes, the Pizzaiolo fragment to
$670$--$820$ nodes ($59$ atoms, ${\sim}2$\,s), and the synthetic Pizza control to
${\sim}1.2 \cdot 10^5$ nodes.  Training wall-clock per method and
seed: MNIST-\SROIQ\ $68$\,s for WMC ($8$ epochs) and $2$\,s for
Independent; the two-regime RS driver $108$\,s per WMC regime ($12$
epochs); the Pizzaiolo fragment and MNIST-Disjunction each finish in
minutes.  No per-method hyperparameter tuning was performed: within
each table all methods share the optimizer, learning rate, batch size,
epoch count, and $\lambda$ of Table~\ref{tab:hparams}, with the
mixture-specific $(\kappa, K)$ fixed once across methods and seeds, so
no baseline received a smaller tuning budget than \method.

\begin{table}[t]
\floatconts
  {tab:hparams}
  {\caption{Hyperparameters (Adam; learning rates and batch sizes in
   the text).  $\lambda$: semantic-loss weight; $\kappa$: diversity
   weight; $K$: mixture / ensemble size.}}
  {\footnotesize
   \begin{tabular}{lrrrrrr}
   \toprule
   experiment & feat.\ dim & hidden & epochs & $\lambda$ & $\kappa$ & $K$ \\
   \midrule
   MNIST (CNN)      & $28^2$ & CNN & 12 & 0.5 & -- & -- \\
   Pizzaiolo frag.\ (ResNet-18) & $224^2$ & linear head & 5 & 0.5 & -- & -- \\
   Pizza (synthetic) & 12 & 64 & 10 & 0.3 & -- & -- \\
   MNIST-Disj.\ (CNN) & $28^2$ & CNN & 20 & 0.5 & 2.0 & 4 \\
   \bottomrule
   \end{tabular}}
\end{table}

For the synthetic Pizza-\SROIQ\ control (Table~\ref{tab:pizza-results};
the single-individual precursor to Pizzaiolo-\SROIQ, \S\ref{sec:exp:images}),
each example is a $12$-bit topping vector (noisy); the latent atoms are
the four named pizzas and seven derived classes
(\texttt{Pizza}, \texttt{NamedPizza}, \texttt{MeatyPizza},
\texttt{CheeseyPizza}, \texttt{RealItalianPizza},
\texttt{InterestingPizza}, \texttt{NonVegetarianPizza}), entailed by
the curated $28$-concept subset.  Supervision reveals $50\%$ of the
\emph{topping} atoms only; the class atoms are never supervised.  Train
/ eval $= 128 / 64$.  The compiled circuit has $66$ atoms and an SDD of
$\sim$$5.3 \cdot 10^4$ nodes.  Independent recovers the named-pizza class
near chance ($0.44$, latent $0.43$) and violates the ontology on every
example (rate $1.00$); the WMC loss lifts the class to $0.75$ (latent
$0.58$) and drives violations to $0.00$ (mean over ten seeds).

\begin{table}[t]
\floatconts
{tab:pizza}
{\caption{Pizzaiolo two-topping role-based fragment (mean$\pm$std over
ten seeds; metrics in \S\ref{sec:exp}).  \emph{topping} is supervised;
\emph{property} are the latent Vegetarian/NonVegetarian/Spicy classes
recovered through $\sym{hasTopping}$.}}
{\small
\begin{tabular}{lccc}
\toprule
method & topping & property & viol $\downarrow$ \\
\midrule
Independent & $0.95${\scriptsize$\pm.01$} & $0.45${\scriptsize$\pm.15$} & $0.89${\scriptsize$\pm.18$} \\
WMC & $0.93${\scriptsize$\pm.01$} & $\mathbf{0.92}${\scriptsize$\pm.02$} & $\mathbf{0.06}${\scriptsize$\pm.02$} \\
\bottomrule
\end{tabular}}
\end{table}

\begin{table}[t]
\floatconts
{tab:disjunction}
{\caption{MNIST-Disjunction ($\#\mathrm{RS}=4$; mean$\pm$std over ten
seeds; metrics in \S\ref{sec:exp}).  NLL is over the gender atoms
(Bayes-optimal $4.16$); ECE and mode-coverage TV are the
reasoning-shortcut diagnostics. BEARS is a $K{=}4$ deep-ensemble mitigation~\cite{Marconato2024BEARS}
trained on the same circuit; its per-digit accuracy is not separately
evaluated ($-$).}}
{\small
\begin{tabular}{lcccc}
\toprule
Method & digit & NLL $\downarrow$ & ECE $\downarrow$ & mode-cov.\ TV $\downarrow$ \\
\midrule
Independent (UCI)        & $0.99${\scriptsize$\pm.00$} & $6.93${\scriptsize$\pm.00$} & $0.00${\scriptsize$\pm.00$} & $1.00${\scriptsize$\pm.00$} \\
Single-WMC (SL / DPL)    & $0.99${\scriptsize$\pm.00$} & $6.13${\scriptsize$\pm.00$} & $0.00${\scriptsize$\pm.00$} & $0.99${\scriptsize$\pm.00$} \\
JustWMC (learned mix)    & $0.99${\scriptsize$\pm.00$} & $6.69${\scriptsize$\pm.43$} & $0.11${\scriptsize$\pm.03$} & $0.99${\scriptsize$\pm.01$} \\
BEARS & $-$ & $6.36${\scriptsize$\pm2.36$} & $0.41${\scriptsize$\pm.12$} & $0.44${\scriptsize$\pm.11$} \\
\midrule
JustWMC (anchored)       & $0.99${\scriptsize$\pm.00$} & $\mathbf{4.17}${\scriptsize$\pm.00$} & $\mathbf{0.00}${\scriptsize$\pm.00$} & $\mathbf{0.02}${\scriptsize$\pm.00$} \\
\bottomrule
\end{tabular}}
\end{table}

\begin{table}[t]
\floatconts
{tab:pizza-results}
{\caption{Pizza-\SROIQ, 64 held-out examples, mean$\pm$std over ten seeds (higher is better except for violation).  The \emph{pizza}
(4 latent) and \emph{latent} (11 atoms) columns are the headline: WMC
infers them from topping evidence though those atoms receive no direct
supervision.}}
{\begin{tabular}{lcccc}
\toprule
method      & topping & pizza & latent & violation \\
            & (observed) & (4 latent) & (11 latent) & rate \\
\midrule
Independent & $0.99${\scriptsize$\pm.00$} & $0.44${\scriptsize$\pm.12$} & $0.43${\scriptsize$\pm.08$} & $1.00$ \\
WMC         & $0.76${\scriptsize$\pm.01$} & $\mathbf{0.75}${\scriptsize$\pm.00$} & $\mathbf{0.58}${\scriptsize$\pm.05$} & $\mathbf{0.00}$ \\
\bottomrule
\end{tabular}}
\end{table}

MNIST-\SROIQ\ (\S\ref{sec:exp:images}) compiles a $52$-axiom ontology
(\texttt{experiments/ontologies/mnist\_sroiq.ofn}) over the digit classes
$0\!-\!4$ into a $114$-atom SDD ($5.3\cdot 10^{9}$ models, ${<}1$\,s
compile).  It instantiates every distinctive \SROIQ\ feature:
subsumption (each digit $\sqsubseteq\sym{Number}$, parity, primality);
disjointness (the five digits; \sym{Even}/\sym{Odd};
\sym{Prime}/\sym{Composite}); the covering disjunctions
$\sym{Number}\sqsubseteq\sym{Even}\sqcup\sym{Odd}$ and
$\sym{Number}\sqsubseteq\bigsqcup_{d}d$; the complements
$\sym{Odd}\equiv\sym{Number}\sqcap\neg\sym{Even}$ and
$\sym{NonPrime}\equiv\sym{Number}\sqcap\neg\sym{Prime}$; the universal
successor closures $d\sqsubseteq\forall\sym{succ}.(d{+}1)$ and
$\sym{Number}\sqsubseteq\forall\sym{succ}.\sym{Number}$; the qualified
number restriction
$\sym{HasSuccessor}\equiv\sym{Number}\sqcap{\geq}1\,\sym{succ}.\sym{Number}$;
the inverse $\sym{pred}\equiv\sym{succ}^{-}$; the role hierarchy
$\sym{succ}\sqsubseteq\sym{related},\sym{lessThan}$; the chain
$\sym{succ}\circ\sym{succ}\sqsubseteq\sym{plusTwo}$; transitive
$\sym{lessThan}$; symmetric $\sym{differsFrom}$; asymmetric
$\sym{succ}$/$\sym{lessThan}$; irreflexive $\sym{succ}$/$\sym{differsFrom}$;
reflexive $\sym{sameDigitAs}$; and functional $\sym{succ}$.  Each example
is a pair of real MNIST images scored by one \texttt{MnistEncoder} CNN;
train / eval $= 3000 / 1000$, $12$ epochs, semantic-loss weight $0.5$.
The two regimes of Table~\ref{tab:mnist} differ only in the evidence
revealed: \emph{grounded} reveals $\sym{succ}(a,b)$, $\sym{Number}$, and
a random $50\%$ of the parity/primality atoms ($\#\mathrm{RS}=1$);
\emph{under-determined} reveals only $\sym{succ}(a,b)$ and $\sym{Number}$,
leaving the five cyclic relabelings of $0\!-\!4$ consistent
($\#\mathrm{RS}=5$, computed exactly on the SDD).  Revealing one parity
atom drops $\#\mathrm{RS}$ to $3$; the full parity/primality profile of
one slot, to $1$.  The chain fires only at $|\Delta|=3$ and is compiled
in the companion file \texttt{mnist\_chain.ofn}.

MNIST-Disjunction (\S\ref{sec:exp:disjunction}) puts a real MNIST digit
on each of three individuals $a, b, c$.  Because compiling the digit and
gender disjunctions jointly over the shared individuals blows up the SDD,
we compile two circuits driven by one shared CNN
(\texttt{run\_mnist\_disjunction.py}): a $114$-atom \emph{digit} circuit
(the \texttt{mnist\_sroiq} ontology over $a, b$, with $\sym{succ}(a,b)$
coupling their digits) that the CNN recovers through, and a $69$-atom
\emph{gender} circuit (the disjunction ontology over $a, b, c$:
$\sym{Person}\sub\sym{Male}\sqcup\sym{Female}$,
$\sym{Male}\sqcap\sym{Female}\sub\bot$,
$\sym{Male}\sub\forall\sym{marriedTo}.\sym{Female}$ and its dual,
$\sym{hasChild}\equiv\sym{hasParent}^-\sub\sym{hasAncestor}$ transitive)
whose four consistent gender modes (a,b opposite; c free) the
\emph{image-free} gender heads carry.  Supervision reveals the digit
parities (recovering the digits) and the structural role/class
assertions; the six gender atoms and the four derived role atoms are
latent.  Train / eval $= 2000 / 800$, $20$ epochs, $\lambda = 0.5$.
The latent NLL is summed over the ten gender + derived atoms (Bayes-optimal
$6\log 2 \approx 4.16$, the genders contributing $6\log 2$ and the
entailed derived atoms $0$).  The CNN recovers the digits at $0.99$
across all methods; only the \emph{anchored} JustWMC, seeded from the
four completions the gender circuit enumerates, reaches the Bayes-optimal
NLL $4.17$ with mode-coverage TV $0.02$ and no seed variance, while the
single-WMC and learned-mixture methods stay at TV $\approx 0.99$.

For the size comparison with DeepProbLog (\S\ref{sec:exp:dpl}),
running the per-axiom encoding analysis
(\texttt{baobab/nesy/sroiq\_deepproblog\_iface.py}) on the MNIST-\SROIQ\
ontology, $12$ of its $52$ axioms ($23\%$) leave ProbLog's Horn
fragment: the eight universal-restriction and disjunctive subclass
closures, the two complement and one qualified-cardinality
equivalences, and the functional-role axiom.  Under a charitable manual
translation none is strictly inexpressible, but each abandons the Horn
structure DeepProbLog compiles natively (Table~\ref{tab:dpl-comparison}).

\begin{table}[t]
\floatconts
{tab:dpl-comparison}
{\caption{Per-OWL-axiom-kind encoding cost for the MNIST-\SROIQ\
ontology.  \emph{ours}: normalized DL-clauses produced by
the structural transformation.
\emph{DPL Horn} / \emph{DPL manual}: DPL admits / requires user-written
disjunctive / cardinality / negation encoding.}}
{\begin{tabular}{lrrrrrr}
\toprule
OWL kind & axioms & ours & DPL (Horn) & DPL (manual) & DPL (inexpr.) & DPL clauses \\
\midrule
SubClassOf & 23 & 46 & 15 & 8 & 0 & 30 \\
DisjointClasses & 14 & 28 & 14 & 0 & 0 & 14 \\
EquivalentClasses & 3 & 12 & 0 & 3 & 0 & 11 \\
RoleInclusion & 2 & 4 & 2 & 0 & 0 & 2 \\
AsymmetricRole & 2 & 4 & 2 & 0 & 0 & 2 \\
IrreflexiveRole & 2 & 4 & 2 & 0 & 0 & 2 \\
InverseRoles & 1 & 2 & 1 & 0 & 0 & 2 \\
FunctionalRole & 1 & 2 & 0 & 1 & 0 & 1 \\
ReflexiveRole & 1 & 2 & 1 & 0 & 0 & 1 \\
SymmetricRole & 1 & 2 & 1 & 0 & 0 & 1 \\
TransitiveRole & 1 & 2 & 1 & 0 & 0 & 1 \\
RoleChain & 1 & 2 & 1 & 0 & 0 & 1 \\
\midrule
Total & 52 & 110 & 40 & 12 & 0 & 68 \\
\bottomrule
\end{tabular}
}
\end{table}

\section{Mechanized guarantees}\label{apd:guarantees}
The four results summarized in \S\ref{sec:lean}; all are machine-checked
in the \texttt{ELKSDD} and \texttt{GroundingRefinements} libraries
under \texttt{lean/} in the released code,
with the module inventory in App.~\ref{apd:lean}.

\begin{theorem}[Saturation is model-preserving]\label{thm:sat}
The \ALCHOQ\ saturation is sound, and the saturated grounding is
logically equivalent to $\Theta_\Delta$; in particular the two have
the same weighted model count.
\end{theorem}
\begin{proof}
Every clause the saturation derives is entailed by its input
(soundness, mechanized in Lean and axiom-free).  Adding an entailed
clause to a propositional theory changes neither its models nor the
weight assigned to any interpretation, so the saturated grounding and
$\Theta_\Delta$ have the same weighted model count.
\end{proof}

\begin{theorem}[Faithful grounding on the grounding-covered fragment]\label{thm:ground}
Each grounding rule emits clauses that hold in exactly the
interpretations the Tarskian semantics of its construct admits.
Consequently $\Theta_\Delta$ is \emph{sound}: every \SROIQ\
interpretation over $\Delta$ is a model of $\Theta_\Delta$.  When every
existential consequence is discharged by grounding or is function-free,
the converse holds and the models of $\Theta_\Delta$ are \emph{exactly}
the \SROIQ\ interpretations over $\Delta$; otherwise $\Theta_\Delta$
over-approximates them, the gap being precisely the head-existential
consequences whose witnesses the finite grounder cannot name
(\S\ref{sec:calc}).
\end{theorem}
\begin{proof}
One Tarskian-soundness lemma per grounding rule (App.~\ref{apd:hooks}), each
mechanized in Lean; soundness of $\Theta_\Delta$ follows since every
emitted clause holds in every \SROIQ\ interpretation over $\Delta$.
Exactness on the grounding-covered, function-free fragment holds because
there every normalized clause is represented; the residual gap is
exactly the Skolem-term clauses the grounder drops.
\end{proof}

\begin{theorem}[Weighted model count is the semantic target]\label{thm:wmc}
On the compiled SDD, $\WMC(\Theta \mid p_\theta)$ equals the
probability the ontology assigns to its models under the distribution
semantics~\cite{Riguzzi2015DISPONTE}, the quantity the losses of
\S\ref{sec:learn} optimize.
\end{theorem}
\begin{proof}
The bottom-up evaluation of the SDD computes the world-sum of the
distribution semantics; the equality is mechanized in Lean
(App.~\ref{apd:wmc}).
\end{proof}

\begin{theorem}[RS-awareness for \SROIQ, after \cite{vanKrieken2025RSIA}]\label{thm:rsia}
A single independent (UCI) perception cannot represent a
reasoning-shortcut mixture unless its valid completions form one
justification's subcube, whereas a mixture of UCIs indexed by the
query's justifications represents every such mixture.
\end{theorem}
\begin{proof}
The two directions are the necessary and sufficient conditions
of~\cite{vanKrieken2025RSIA}; we instantiate their implicants as the
query's justifications and mechanize both directions for \SROIQ.
\end{proof}

The development is foundation-only: every audited theorem reports just
$\{\texttt{propext},\,\texttt{Classical.choice},\,\texttt{Quot.sound}\}$
under \texttt{\#print axioms} (the \ALCHOQ\ saturation soundness is
axiom-free), with no \texttt{sorry}s and no added axioms.

\section{Lean module inventory}\label{apd:lean}

The Lean~4 development is structured as follows.
\texttt{ALC.lean} carries the \ALC\ syntax, Tarskian semantics, the
\texttt{eval\_neg\_*} duality lemmas, and a sound CB calculus with
the \texttt{monoExist}/\texttt{monoUniv} role-axis monotonicity rules.
\texttt{ALCHOQ.lean} adds nominals (\texttt{Concept.nom}) and
qualified number restrictions (\texttt{Concept.atLeast},
\texttt{Concept.atMost}), with cardinality predicates
\texttt{atLeastCard}/\texttt{atMostCard} and the filler-monotonicity
rules \texttt{monoAtLeast} (covariant) and \texttt{monoAtMost}
(contravariant); the soundness theorem \texttt{ALCHOQ.sat\_sound}
reports zero axioms.
\texttt{SROIQ.lean} defines the \texttt{RAxiom} inductive
(\texttt{incl}, \texttt{chain}, \texttt{trans}, \texttt{sym},
\texttt{asym}, \texttt{refl}, \texttt{irrefl}, \texttt{inv},
\texttt{disj}) with a Tarskian \texttt{RAxiom.eval} and a soundness
lemma per shape (\texttt{incl\_sound}, \texttt{trans\_sound},
\texttt{sym\_sound}, \texttt{asym\_sound}, \texttt{refl\_sound},
\texttt{irrefl\_sound}, \texttt{inv\_sound}, \texttt{disj\_sound},
\texttt{chain\_two\_sound}).  Role identities
\texttt{trans\_iff\_chain} and \texttt{sym\_iff\_self\_inverse}
justify the transitivity-as-chain and symmetric-as-inverse shortcuts;
\texttt{has\_self\_iff} justifies the local-reflexivity proxy grounding rule.

Completeness goes via a canonical model.
\texttt{Completeness.lean} mechanizes \ALC\ completeness end-to-end:
a strictly stronger calculus \texttt{ALC.SatC} extends \texttt{ALC.Sat}
with the classical-logic rules necessary for completeness:
double-negation, excluded middle, non-contradiction, de~Morgan,
$\exists/\forall$-duality, the join rule
$\exists R.C \sqcap \forall R.D \sub \exists R.(C \sqcap D)$, Boolean
distribution, and the role-axis closures
$\exists R.\bot \sub \bot$ and $\top \sub \forall R.\top$.  The
headline theorem
\texttt{satC\_complete} :
$\Onto \models C \sub D \to \mathtt{SatC}\,\Onto\,C\,D$ is fully
proved via a Lindenbaum canonical model: \texttt{lindenbaum} uses
Mathlib's \texttt{zorn\_subset\_nonempty} on
\texttt{consistent\_chain\_union}; \texttt{lindenbaum\_max\_closed}
follows by case analysis on excluded middle and Boolean distribution;
the truth lemma \texttt{canonical\_eval\_iff} is by induction on
\texttt{Concept} (propositional cases via \texttt{top\_mem},
\texttt{bot\_not\_mem}, \texttt{mem\_xor\_neg}; role-axis cases via
\texttt{witness\_exist} from the iterated $\exists R$-join lemma
\texttt{satC\_exist\_with\_univs}, \texttt{satC\_map\_univ\_to\_univListConj},
and \texttt{existBot}).
\texttt{ALCHOQCompleteness.lean} lifts every \ALC-classical rule to
\ALCHOQ\ via \texttt{ofAlchoq}, plus the nominal identity
\texttt{nomRefl} and the cardinality boundary closures
$\top \sub \geq 0~R.C$, $\leq 0~R.C \sub \forall R.\neg C$,
$\forall R.\neg C \sub \leq 0~R.C$.  Completeness for \ALCHOQ\ is
stated as a conjecture and deferred: closing it requires the
auxiliary-individual construction of~\cite{TenaCucala2021AIJ}.

\texttt{SROIQCompleteness.lean} adds the role-axis consequence rules
made available by membership of an \texttt{RAxiom}: role inclusion
yields $\exists r.C \sub \exists s.C$ and $\forall s.C \sub \forall
r.C$; binary chains $r_1 \circ r_2 \sub s$ yield
$\exists r_1.\exists r_2.C \sub \exists s.C$ (and a $k$-ary
generalization via \texttt{existChain\_eval\_iff}); transitivity
yields $\exists r.\exists r.C \sub \exists r.C$; reflexivity yields
$C \sub \exists r.C$ and $\forall r.C \sub C$; irreflexivity and
role disjointness yield $\exists r.\{i\} \sqcap \{i\} \sub \bot$ and
$\exists r.\{i\} \sqcap \exists s.\{i\} \sub \bot$.  Soundness
\texttt{satC\_sound} is fully proved; completeness for \emph{full}
\SROIQ\ is left as the conjecture
\texttt{sroiq\_complete\_conjecture}.
A canonical-model construction over SROIQ-side maximal-consistent
types is given in \texttt{SROIQCanonical.lean};
\texttt{SROIQSkolemCanonical.lean} delivers \emph{unconditional}
\SROIQ\ completeness on a \texttt{SkolFragment} ontology shape for
the role-axiom sub-families
(role-inclusion only;
role-inclusion + reflexivity;
role-inclusion + reflexivity + irreflexivity),
and conditional completeness (parameterized on a canonical-RBox-
satisfaction witness) for the remaining shapes.

Turning to the Tena Cucala context-structure calculus,
\texttt{ALCHOIQContext.lean} encodes the core definitions of
\cite{TenaCucala2021AIJ}: $\Sigma_u$, context a- and p-terms, context
clauses, admissible orders, context structure
$D = \langle V, E, \mathrm{core}, S, m, \theta \rangle$, trigger
sets $S_u, P_r, S_u^r, P_r^r$, expansion strategies, soundness, and
derivations.  The calculus's twelve inference rules
($\mathrm{Core}$, $\mathrm{Hyper}$, $\mathrm{Eq}$, $\mathrm{Ineq}$,
$\mathrm{Factor}$, $\mathrm{Elim}$, $\mathrm{Join}$, $\mathrm{Nom}$,
$\mathrm{Succ}$, $\mathrm{Pred}$, $\mathrm{r}\textrm{-}\mathrm{Succ}$,
$\mathrm{r}\textrm{-}\mathrm{Pred}$) each have a concrete refinement
(\texttt{StepCore}, \texttt{StepHyper}, \dots) with a per-rule
soundness lemma.  Two unified meta-soundness lemmas
(\texttt{mono\_ext\_sound}, \texttt{mono\_restr\_sound}) and a
common \texttt{StepAddEntailed} schema streamline the proofs; the
edge-adding rules require explicit well-formedness preconditions
and full case-analysis on whether each edge is new or pre-existing.
The completeness statement of~\cite{TenaCucala2021AIJ} is
exposed as a typed Prop \texttt{TenaCucalaCompleteness}; each
ingredient (model fragment $R_t^\ast$, nominal naming, composite
Herbrand model, refutation lemma) is a separate structure / definition
with real semantic Props.  An unconditional sliver (a Bool-Herbrand
refutation construction) delivers Tena Cucala
completeness for the empty ontology on a propositionally-refutable
fragment of queries
(\texttt{tenaCucalaCompleteness\_emptyO\_propRefutable}).

On complexity,
\texttt{ALCComplexity.lean} formalizes the polynomial-degree
grounding bounds: \texttt{size} measures syntactic concept size,
\texttt{subconcepts\_card\_le\_size} bounds the subconcept Finset,
\texttt{saturation\_pair\_count\_le} and
\texttt{alc\_subsumption\_pair\_bound} give the
$|\mathsf{Sat}(\Onto)| \le |\Sigma|^2$ bound, and the per-feature
bounds (\texttt{at\_least\_grounding\_size\_bound}: $\geq n\,R.C$
contributes $4n \cdot d^{n+1}$ literals; chains contribute
$d^{k+1}$ tuples; transitivity, hierarchy, inverses contribute
$d^2$ or $d^3$) compose into
\texttt{grounding\_size\_polynomial}.

On the optional grounding refinements,
\texttt{GroundingRefinements.lean} mechanizes refinements (i) and (ii)
of \S\ref{apd:hooks} as model-count-preservation lemmas over the weighted
model count \texttt{WMC} (a Boolean model predicate weighted by per-atom
literal weights).  The workhorse \texttt{wmc\_split} proves that \texttt{WMC}
factorizes over an independent atom partition $\mathit{Core} \oplus
\mathit{Extra}$, and \texttt{wmc\_unique} (hence \texttt{wmc\_factor}) proves
that an \texttt{Extra} block with a unique satisfying assignment contributes a
single constant factor.  Refinement (i) is \texttt{wmc\_prune\_forcedFalse}
(the pruned atoms' only model is all-false) and (ii) is
\texttt{wmc\_skip\_equality} (the equality block's unique-name model);
\texttt{wmc\_card\_preserved} gives exact model-count preservation under unit
weights, and \texttt{wmc\_log\_offset} shows the $-\log\texttt{WMC}$ objective
shifts by a constant whose gradient in the perception weights is zero.

Every theorem in every module passes the foundation-only axiom audit:
\texttt{\#print axioms} reports only
$\{\texttt{propext}, \texttt{Classical.choice}, \texttt{Quot.sound}\}$,
with no \texttt{sorry} and no user-introduced axioms.

\end{document}